\documentclass{article}

\usepackage{fullpage}
\usepackage{microtype}
\usepackage{graphicx}
\usepackage{subfigure}
\usepackage{booktabs} % for professional tables

\usepackage[breaklinks=true]{hyperref}
\hypersetup{colorlinks,
	linkcolor=blue,
	citecolor=blue,
	urlcolor=magenta,
	linktocpage,
	plainpages=false}

\usepackage{natbib}
\usepackage{amsmath}
\usepackage{amsthm}
\usepackage{amssymb}
\usepackage{mathtools}
\usepackage{tikz}
\usepackage{xcolor}
\usetikzlibrary{arrows}
\usepackage{dsfont}
\usepackage{comment}
\usepackage{enumerate}
\usepackage{caption}
\usepackage{breakcites}
\newcommand{\calN}{\mathcal{N}}
\newcommand{\calD}{\mathcal{D}}
\newcommand{\calR}{\mathcal{R}}

\newcommand{\sgn}{\mathrm{sign}}

\newcommand{\poly}{\mathrm{poly}}

\newcommand{\E}{\mathop\mathbb{E}}

\newcommand{\tr}{\mathrm{tr}}

\newcommand{\unif}{\mathrm{unif}}

\def\R{\mathbb{R}}

\def\cD{\mathcal{D}}

\def\cF{\mathcal{F}}

\def\cN{\mathcal{N}}

\def\vw{\mathbf{w}}
\def\vW{\mathbf{W}}
\def\va{\mathbf{a}}
\def\vZ{\mathbf{Z}}
\def\vX{\mathbf{X}}
\def\vx{\mathbf{x}}
\def\vy{\mathbf{y}}
\def\vbeta{\bm{\beta}}

\newcommand{\mat}[1]{\mathbf{#1}}
\newcommand{\vect}[1]{\mathbf{#1}}
\newcommand{\norm}[1]{\left\|#1\right\|}
\newcommand{\abs}[1]{\left|#1\right|}
\newcommand{\expect}{\mathbb{E}}

\newcommand{\indict}{\mathbb{I}}

\newcommand{\vectorize}[1]{\text{vec}\left(#1\right)}

\newcommand{\relu}[1]{\sigma\left(#1\right)}

\newcommand{\useq}{\xi}

\usepackage{bm}
\newcommand{\bone}[1]{\indict\left\{#1\right\}}
\newcommand{\eps}{\varepsilon}

\newcommand{\reals}{\mathbb{R}}

\newcommand{\inp}[2]{\left\langle #1,#2\right\rangle}

\newcommand{\Erc}[1]{\E_{{\bm\eps} \sim \{\pm 1\}^n }\left[ #1 \right]}

\newtheorem{thm}{Theorem}[section]

\newtheorem{lem}[thm]{Lemma}
\newtheorem{cor}[thm]{Corollary}

\newtheorem{defn}{Definition}[section]

\newtheorem{rem}{Remark}[section]
\newtheorem{example}{Example}[section]

\newtheorem{question}{Question}

\allowdisplaybreaks

\newcommand{\simon}[1]{\textcolor{blue}{[Simon: #1]}}

\title{Fine-Grained Analysis of Optimization and Generalization for Overparameterized Two-Layer Neural Networks}

\date{}

\author{Sanjeev Arora\thanks{Princeton University and Institute for Advanced Study. Email: \texttt{arora@cs.princeton.edu}}
	\and 
	Simon S. Du\thanks{Carnegie Mellon University. Email: \texttt{ssdu@cs.cmu.edu}}
	\and
	Wei Hu\thanks{Princeton University. Email: \texttt{huwei@cs.princeton.edu}}
	\and
	Zhiyuan Li\thanks{Princeton University. Email: \texttt{zhiyuanli@cs.princeton.edu}}
	\and
	Ruosong Wang\thanks{Carnegie Mellon University. Email: \texttt{ruosongw@andrew.cmu.edu}}
}

\begin{document}

\maketitle

\begin{abstract}
	
Recent works have cast some light on the mystery of why deep nets fit any data and generalize despite being very overparametrized. This paper analyzes training and generalization for a simple $2$-layer ReLU net with random initialization, and provides the following improvements over
recent works:

 \begin{enumerate}[(i)]
 \item Using a tighter characterization of training speed than recent papers, an explanation for why training a neural net with random labels leads to slower training, as originally observed in  [Zhang et al. ICLR'17]. 
 \item Generalization bound independent of network size, using a data-dependent complexity measure. 
  Our measure distinguishes clearly between random labels and true labels on MNIST and CIFAR, as shown by experiments. 
  Moreover, recent papers require sample complexity to increase (slowly) with the size, while our sample complexity is completely independent of the network size.

 \item Learnability of a broad class of smooth functions %somewhat larger than in earlier works
  by $2$-layer ReLU nets trained via gradient descent.
 \end{enumerate}
 
The key idea is to track dynamics of training and generalization via properties of a related kernel. 
 
\iffalse 
It was recently proved that over-parameterized neural networks with sufficiently large width trained with randomly initialized gradient descent (GD) can achieve zero training loss on any non-degenerate data.
However, this result holds even for arbitrary labels (including random labels) and thus cannot explain the generalization behavior of solutions found by GD.
In this work, we give a fine-grained analysis to characterize the optimization and generalization performance of two-layer over-parameterized ReLU activated neural networks trained with GD. This analysis  \begin{enumerate}[(i)]
	\item explains why training a neural network with random labels is slower than with true labels, as observed in~\cite{zhang2016understanding}, 
	\item yields a generalization bound for the solution found by GD, which is independent of the number of parameters in the network and only depends on a data-dependent complexity measure that can clearly distinguish random labels and true labels on MNIST and CIFAR datasets, and
	\item implies that two-layer neural nets can provably learn a broad class of smooth functions via GD. %, including linear functions, even-degree polynomials and cosine function.
\end{enumerate}
\fi 
%Our proof is through a fine-grained analysis of the trajectory of GD for optimizing over-parameterized two-layer neural networks, and we believe that this approach may be useful for analyzing the generalization behavior of deep neural networks.

%\simon{notes:
%	\begin{itemize}
%		\item Fix reference typos.
%	\end{itemize}
%	}
\end{abstract}

\section{Introduction}
\label{sec:intro}

%motivate the problem from existing empirical study

The well-known work of \citet{zhang2016understanding} highlighted intriguing experimental phenomena about deep net training -- specifically, optimization and generalization -- and asked whether theory could explain them. They showed that sufficiently powerful nets (with vastly more parameters than number of training samples) can attain zero training error, regardless of whether the data is properly labeled or randomly labeled. Obviously, training with randomly labeled data cannot generalize, whereas training with properly labeled data generalizes. See Figure~\ref{fig:generalization} replicating some of these results. 

Recent papers have begun to provide explanations, showing that gradient descent can allow an overparametrized multi-layer  net to attain arbitrarily low training error on fairly generic datasets~\citep{du2018global,du2018provably,li2018learning,allen2018convergence,zou2018stochastic}, provided the amount of overparametrization is a high polynomial of the relevant parameters (i.e. vastly more than the overparametrization in \cite{zhang2016understanding}). Under further assumptions it can also be shown that the trained net generalizes~\citep{allen2018learning}.  But some issues were not addressed in these papers, and the goal of the current paper is to address them. 

%convergence rate is different
First, the experiments in \cite{zhang2016understanding} show that though the nets attain zero training error on even random data, the convergence rate is much slower. See Figure~\ref{fig:rate_mnist}. % and~\ref{fig:rate_cifar}.
%A natural question is how to explain this behavior.
\begin{question}\label{ques:conv}
	Why do true labels give faster convergence rate than random labels for gradient descent?
%	Optimization is faster when using true labels than when using random labels. 
\end{question}
%Specifically, on CIFAR and ImageNet, if we replace some or all labels with random labels, the convergence to zero training error occurse slower. See Figure~\ref{fig:rate_mnist} and~\ref{fig:rate_cifar}. 
The above papers do not answer this question, since their proof of convergence does not distinguish between good and random labels.

%implict regularization
The next issue is about generalization: clearly, some property of properly labeled data controls generalization, but what? 
Classical measures used in generalization theory such as VC-dimension and Rademacher complexity are much too pessimistic.
%as complexity measure, %which acts a surrogate of the generalization error.
%but it cannot explain this phenomenon because these neural networks have more parameters than the number of data points.
%Several attempts have been proposed to attack this problem.
A line of research proposed norm-based (e.g. \cite{bartlett2017spectrally}) and compression-based bounds~\citep{arora2018stronger}.
But the sample complexity upper bounds obtained are still far too weak. 
Furthermore they rely on some property of the trained net that is \emph{revealed/computed} at the end of training. There is no property of data alone that determine upfront  whether the trained net will generalize.
%Norm-based measures seem appealing, but since there is no explicit regularization, one cannot control the  of learned neural networks in a straightforward manner, making it difficult to analyze generalization.
%Based on these empirical findings, \citet{zhang2016understanding} argued ``understanding deep learning requires rethinking generalization."
%\begin{obs}\label{obs:gen}
%Simple first-order optimization methods usually find generalizable solutions given true labels.
%\end{obs}
%A line of research proposed norm-based generalization bounds~\citep{bartlett2002rademacher,bartlett2017spectrally,neyshabur2015norm,neyshabur2017pac,neyshabur2018towards,konstantinos2017pac,golowich2017size,li2018tighter} and compression-based bounds~\citep{arora2018stronger}.
%\citet{dziugaite2017computing,zhou2018nonvacuous} used the PAC-Bayes approach to compute non-vacuous generalization bounds for MNIST and ImageNet, respectively.
A recent paper~\citep{allen2018learning} assumed that there exists an underlying (unknown) neural network that achieves low error on the data distribution,
and the amount of data available is quite a bit more than the minimum number of samples needed to learn this underlying neural net. Under this condition, the overparametrized net (which has way more parameters) can learn in a way that generalizes. However, it is hard to verify from data whether this assumption is satisfied, even after the larger net has finished training.\footnote{In Section~\ref{sec:rel}, we discuss the related works in more details.} Thus the assumption is in some sense unverifiable. 
 
\begin{question}\label{ques:gen}
	Is there an easily verifiable complexity measure that can differentiate true labels and random labels?
\end{question}

%In fact, without explicit regularization, there is still one source left that can be the reason behind generalization: the training algorithms.
%To understand this phenomenon, 
Without explicit regularization, to attack this problem,
one must resort to algorithm-dependent generalization analysis.
One such line of work established that first-order methods can automatically find minimum-norm/maximum-margin solutions that fit the data in the settings of logistic regression, deep linear networks, and symmetric matrix factorization~\citep{soudry2018implicit,gunasekar2018characterizing,gunasekar2018implicit,ji2018gradient,li2018algorithmic}.
However, how to extend these results to non-linear neural networks remains unclear~\citep{wei2018margin}.
Another line of algorithm-dependent analysis of generalization \citep{hardt2015train,mou2017generalization,chen2018stability} used stability of specific optimization algorithms  that satisfy certain generic properties like convexity, smoothness, etc.
However, as the number of epochs becomes large, these generalization bounds are vacuous. 
%Therefore, to understand generalization in deep learning one needs to exploit specific properties of the neural network in question
%instead of just relying on generic properties.

%In this paper we give the first analysis on the implicit regularization imposed by the gradient descent algorithm on non-linear neural networks.

\paragraph{Our results.}
We give a new analysis that provides answers to Questions~\ref{ques:conv} and~\ref{ques:gen} for overparameterized two-layer neural networks with ReLU activation trained by gradient descent (GD), when the number of neurons in the hidden layer is sufficiently large.
In this setting, \citet{du2018provably} have proved that GD with random initialization can achieve zero training error for any non-degenerate data.
We give a more refined analysis of the trajectory of GD which enables us to provide answers to Questions~\ref{ques:conv} and~\ref{ques:gen}.
In particular:
\begin{itemize}
\item In Section~\ref{sec:rate}, using the trajectory of the network predictions on the training data during optimization, we accurately estimate the magnitude of training loss in each iteration.
%on predictions $\{\vect{u}(k)\}_{k=1}^\infty$, we are able to explain Observation~\ref{obs:conv}.
Our key finding is that the number of iterations needed to achieve a target accuracy depends on the projections of data labels on the eigenvectors of a certain Gram matrix to be defined in Equation~\eqref{eqn:H_infy_defn}.
On MNIST and CIFAR datasets, we find that such projections are significantly different for true labels and random labels, and as a result we are able to answer Question~\ref{ques:conv}.
%See Section~\ref{sec:rate}.

\item In Section~\ref{sec:generalization}, we give a generalization bound for the solution found by GD, based on accurate estimates of how much the network parameters can move during optimization (in suitable norms).
Our generalization bound depends on a \emph{data-dependent complexity measure} (c.f. Equation~\eqref{eqn:complexity}), and notably, is completely independent of the number of hidden units in the network.
Again, we test this complexity measure on MNIST and CIFAR, and find that the complexity measures for true and random labels are significantly different, which thus answers Question~\ref{ques:gen}.
%To our knowledge, this is the first algorithm-based generalization theory for two-layer neural networks, thus answering an open question in \cite{wei2018margin}.

Notice that because zero training error is achieved by the solution found by GD, a generalization bound is an upper bound on the error on the data distribution (test error).
We also remark that our generalization bound is valid for \emph{any data labels} -- it does not require the existence of a small ground-truth network as in~\citep{allen2018learning}.
Moreover, our bound can be efficiently computed for any data labels. 

\item In Section~\ref{sec:improper}, we further study what kind of functions can be provably learned by two-layer ReLU networks trained by GD. 
Combining the optimization and generalization results, we uncover a broad class of learnable functions, including linear functions, two-layer neural networks with polynomial activation $\phi(z) = z^{2l}$ or cosine activation, etc.
Our requirement on the smoothness of learnable functions is weaker than that in~\citep{allen2018learning}.
\end{itemize}

Finally, we note that the intriguing generalization phenomena in deep learning were observed in kernel methods as well~\cite{belkin2018understand}.
The analysis in the current paper is also related to a kernel from the ReLU activation (c.f.~Equation~\eqref{eqn:H_infy_defn}).

\section{Related Work}
\label{sec:rel}
In this section we survey previous works on optimization and generalization aspects of neural networks.

%landscape does not work
\paragraph{Optimization.}
%For optimization, 
Many papers tried to characterize geometric landscapes of objective functions~\citep{safran2017spurious,zhou2017critical,freeman2016topology,hardt2016identity,nguyen2017loss,kawaguchi2016deep,venturi2018neural,soudry2016no,du2018power,soltanolkotabi2018theoretical,haeffele2015global}.
The hope is to leverage recent advance in first-order algorithms~\citep{ge2015escaping,lee2016gradient,jin2017escape} which showed that if the landscape satisfies (1) all local minima are global and (2) all saddle points are strict (i.e., there exists a negative curvature), then first-order methods can escape all saddle points and find a global minimum.
Unfortunately, these desired properties do not hold even for simple non-linear shallow neural networks~\citep{yun2018critical} or 3-layer linear neural networks~\citep{kawaguchi2016deep}.

%The main problem with the landscape approach is that in order to apply the generic algorithmic results, one needs to show the above two benign geometric properties on the whole parameter space.
%While it is also possible to show these good properties hold locally, e.g., around the initialization region, in order to obtain a complete convergence result, one needs to analyze the region SGD can explore which necessarily requires a  characterization of the trajectory of SGD.
%However, recall the primary goal is to show SGD can find a global minimum, it makes more sense to directly show the trajectory generated by randomly initialized SGD towards a global minimum.
%
%\simon{TO DO: check Nadav's blog and Sanjeev's slides to see if there are anything else to say} 

%recent trajectory based analysis for optimization with strong assumptions
Another approach is to directly analyze trajectory of the optimization method and to show convergence to global minimum.
A series of papers made strong assumptions on input distribution as well as realizability of labels, and showed global convergence of (stochastic) gradient descent for some shallow neural networks~\citep{tian2017analytical,soltanolkotabi2017learning,brutzkus2017globally,du2017convolutional,du2017spurious,li2017convergence}.
Some local convergence results have also been proved~\citep{zhong2017recovery,zhang2018learning}.
However, these assumptions are not satisfied in practice.

%mean field papers
For two-layer neural networks, a line of papers used mean field analysis to establish that for infinitely wide neural networks, the empirical distribution of the neural network parameters can be described as a Wasserstein gradient flow~\citep{mei2018mean,chizat2018global,sirignano2018mean,rotskoff2018neural,wei2018margin}.
However, it is unclear whether this framework can explain the behavior of first-order methods on finite-size neural networks. % can minimize the empirical risk in polynomial time.

%recent breakthroughs
Recent breakthroughs were made in understanding optimization of overparameterized neural networks through the trajectory-based approach.
They proved global polynomial time convergence of (stochastic) gradient descent on non-linear neural networks for minimizing empirical risk.
Their proof techniques can be roughly classified into two categories.
\citet{li2018learning,allen2018convergence,zou2018stochastic} analyzed the trajectory of parameters and showed that on the trajectory, the objective function satisfies certain gradient dominance property.
On the other hand, \cite{du2018global,du2018provably} analyzed the trajectory of network predictions on training samples and showed that it enjoys a strongly-convex-like property.
%Our analysis on the convergence of gradient descent builds upon \cite{du2018provably} but 
%we give a more refined analysis on the trajectory which enables us to explain Observation~\ref{obs:conv}.
%Interestingly, our analysis on generalization relies on characterizing the trajectory of the parameters, using similar ideas in \cite{li2018learning}.

%norm-based bounds
\paragraph{Generalization.}
It is well known that the VC-dimension of neural networks is at least linear in the number of parameters~\citep{bartlett2017nearly}, and therefore classical VC theory cannot explain the generalization ability of modern neural networks with more parameters than training samples.
% where there are more parameters in the neural networks than the number of samples and modern neural networks have VC-dimension that is at least linear in the number of parameters~\citep{bartlett2017nearly}.
Researchers have proposed norm-based generalization bounds~\citep{bartlett2002rademacher,bartlett2017spectrally,neyshabur2015norm,neyshabur2017pac,neyshabur2018the,konstantinos2017pac,golowich2017size,li2018tighter} and compression-based bounds~\citep{arora2018stronger}.
%However, to our knowledge, all these bounds still have dependence on the width of neural networks (see discussions in \cite{golowich2017size}).
\citet{dziugaite2017computing,zhou2018nonvacuous} used the PAC-Bayes approach to compute non-vacuous generalization bounds for MNIST and ImageNet, respectively.
All these bounds are \emph{posterior} in nature -- they depend on certain properties of the \emph{trained} neural networks.
Therefore, %these bounds are not indicative for e.g. neural network architecture selection, because 
one has to finish training a neural network to know whether it can generalize.
Comparing with these results, our generalization bound only depends on training data and can be calculated without actually training the neural network.
%\simon{@zhiyuan: can you double check this paragraph?}

%a priori generalization bound
Another line of work assumed the existence of a true model, and showed that the (regularized) empirical risk minimizer has good generalization with sample complexity that depends on the true model~\citep{du2018many,ma2018priori,imaizumi2018deep}.
%On the other hand, some recent works put assumptions on the complexity of the target function and show the global minimizer of the (regularized) empirical loss function can learn these functions with near optimal statistical convergence rates~\citep{du2018many,ma2018priori,imaizumi2018deep}.
%However, these papers ignore the optimization algorithm and it is not clearly why the target functions in these papers are what actually we want to learn in practice.
These papers ignored the difficulty of optimization, while we are able to prove generalization of the solution found by gradient descent.
Furthermore, our generic generalization bound does not assume the existence of any true model.

%Comparing with these papers, our result is optimization-algorithm-dependent and the bound depends on a general complexity measure instead of the complexity of the underlying target function.

%For algorithm-dependent bounds..\simon{maybe we don't need them because we already mentioned them.}

%yuanzhi's generalization paper
Our paper is closely related to~\citep{allen2018learning} which showed that two-layer overparametrized neural networks trained by randomly initialized stochastic gradient descent can learn a class of infinite-order smooth functions. % that can be decomposed by low-degree polynomials.
In contrast, our generalization bound depends on a data-dependent complexity measure that can be computed for any dataset, without assuming any ground-truth model.
Furthermore, as a consequence of our generic bound, we also show that two-layer neural networks can learn a class of infinite-order smooth functions, with a less strict requirement for smoothness.
%Our result is more general because our generalization bound is depends a general complexity measure and as a direct corollary we show we can learn low-degree polynomials as well.
%Furthermore, our bound is completely independent of the number of hidden units, while there is a poly-logarithmic dependence in \cite{allen2018learning}.
\citet{allen2018learning} also studied the generalization performance of three-layer neural nets.

Lastly, our work is related to kernel methods, especially recent discoveries of the connection between deep learning and kernels~\citep{jacot2018neural,chizat2018note,daniely2016toward,daniely2017sgd}.
%In fact, our generalization bound has the similar form of that of kernel SVM.
%However, our proof does not rely on existing analysis of kernel methods.
%Instead, we take a trajectory-based approach which enables us to analyze optimization and generalization of neural nets at the same time.
%How to make an explicit connection between our approach and kernels is an interesting open problem.
Our analysis utilized several properties of a related kernel from the ReLU activation (c.f.~Equation~\eqref{eqn:H_infy_defn}).

%\simon{Maybe mention other empirical works as well?}
%\simon{expressiveness?}

\section{Preliminaries and Overview of Results}
\label{sec:pre}
%In this section we formally introduce our setting as well as necessary notation and definitions used in this paper.
% introduce necessary notations and background knowledege that we will use in later sections.

\paragraph{Notation.}
We use bold-faced letters for vectors and matrices. For a matrix $\mat A$, let $\mat A_{ij}$ be its $(i, j)$-th entry.
We use $\norm{\cdot}_2$ to denote the Euclidean norm of a vector or the spectral norm of a matrix, and use $\norm{\cdot}_F$ to denote the Frobenius norm of a matrix.
Denote by $\lambda_{\min}(\mat A)$ the minimum eigenvalue of a symmetric matrix $\mat A$.
Let $\vectorize{\mat A}$ be the vectorization of a matrix $\mat A$ in column-first order.
Let $\mat I$ be the identity matrix and $[n]=\{1, 2, \ldots, n\}$.
Denote by $\cN(\bm\mu, \mat\Sigma)$ the Gaussian distribution with mean $\bm\mu$ and covariance $\mat{\Sigma}$.
Denote by $\relu{\cdot}$ the ReLU function $\relu{z} = \max\{z, 0\}$.
Denote by $\indict\{E\}$ the indicator function for an event $E$.

\subsection{Setting: Two-Layer Neural Network Trained by Randomly Initialized Gradient Descent} \label{sec:setup}

We consider a two-layer ReLU activated neural network with $m$ neurons in the hidden layer:
\begin{align*}
f_{\mat{W},\vect{a}}(\vect{x}) = \frac{1}{\sqrt{m}}\sum_{r=1}^{m} a_r \relu{\vect{w}_r^\top \vect{x}},
\end{align*}
where $\vect{x}\in \mathbb{R}^d$ is the input, $\vect{w}_1, \ldots, \vect{w}_m \in \mathbb{R}^d$ are weight vectors in the first layer, $a_1, \ldots, a_m \in \mathbb{R}$ are weights in the second layer.
For convenience we denote $\mat W = (\vect w_1, \ldots, \vect w_m) \in \R^{d\times m}$ and $\vect a = (a_1, \ldots, a_m)^\top \in \R^m$.

%\paragraph{Data and training process.}
We are given $n$ input-label samples $S = \{ (\vect x_i, y_i) \}_{i=1}^n$ drawn i.i.d. from an underlying data distribution $\calD$ over $\R^d\times \R$.
We denote $\mat X = (\vect x_1, \ldots, \vect x_n) \in \R^{d\times n}$ and $\vect y = (y_1, \ldots, y_n)^\top \in \R^n$.
For simplicity, we assume that for $(\vect x, y)$ sampled from $\calD$, we have $\norm{\vect x}_2=1$ and $|y|\le1$.

We train the neural network by \emph{randomly initialized gradient descent (GD)} on the quadratic loss over data $S$.
In particular, we first initialize the parameters randomly:
\begin{equation} \label{eqn:random-init}
\vect{w}_r(0) \sim \calN(\vect{0},\kappa^2\mat{I}), a_r \sim \unif\left(\left\{-1,1\right\}\right), \quad \forall r\in[m],
\end{equation} %\label{eq:init}
where $0<\kappa\le1$ controls the magnitude of initialization, and all randomnesses are independent.
We then fix the second layer $\vect{a}$ and optimize the first layer $\mat{W}$ through GD on the following objective function:
\begin{equation} \label{eqn:objective-function}
\Phi(\mat{W}) = 
%\sum_{i=1}^{n}L_i(\mat{W})\triangleq
\frac{1}{2}\sum_{i=1}^{n}\left(y_i-f_{\mat W, \vect a}(\vect{x}_i)\right)^2.
\end{equation}
The GD update rule can be written as:\footnote{Since ReLU is not differentiable at $0$, we just define ``gradient'' using this formula, and this is indeed what is used in practice.}
\begin{equation*} \label{eqn:w_r-gd}
\begin{aligned}
&\vw_r(k+1) - \vw_r(k)  =  -\eta \frac{\partial \Phi(\mat W(k))}{\partial \vect{w}_r} \\
=\,& -\eta  \frac{a_r}{\sqrt m}  \sum_{i=1}^n (f_{\mat{W}(k),\vect{a}}(\vect{x}_i) - y_i) \indict \left\{ \vect{w}_r(k)^\top \vect x_i \ge 0 \right\}  \vect{x}_i,
\end{aligned}
\end{equation*}
where $\eta>0$ is the learning rate.

%Following \cite{du2018provably}, we initialize $\vect{w}_r(0) \sim N(\vect{0},\kappa^2\mat{I})$ and $a_r \sim \unif\left[\left\{-1,+1\right\}\right]$.
%Here $\kappa$ controls the magnitude of the initialization.
%In this work we fix the second layer $\vect{a}$ and only optimize the first layer through gradient descent where at each iteration we update $\mat{W}$ according to
%\begin{align}
%\mat{W}(k+1) = \mat{W}(k) - \eta \frac{\partial L_i(\mat{W}(k))}{\partial \mat{W}(k)}. \label{eqn:sgd}
%\end{align}

%\begin{align}
%\mat{W}(k+1) = \mat{W}(k) - \eta \nabla L(\mat W(k)). \label{eqn:gd}
%\end{align}

%\subsection{Linear Convergence of GD to Global Minimum}
\subsection{The Gram Matrix from ReLU Kernel}\label{sec:relu_kernel}
Given $\{\vx_i\}_{i=1}^n$, we define the following \emph{Gram matrix} $\mat H^\infty \in \R^{n\times n}$ as follows:
\begin{equation}\label{eqn:H_infy_defn}
\begin{aligned}
\mat{H}_{ij}^\infty &= \expect_{\vect{w} \sim \calN(\vect{0},\mat{I})}\left[ \vect{x}_i^\top \vect{x}_j\indict\left\{\vect{w}^\top \vect{x}_i \ge 0, \vect{w}^\top \vect{x}_j \ge 0\right\}\right] \\
&= \frac{\vect{x}_i^\top\vect{x}_j\left(\pi - \arccos(\vect{x}_i^\top\vect{x}_j)\right)}{2\pi}, \quad \forall i, j\in[n].
\end{aligned}
\end{equation}
This matrix can be viewed as a Gram matrix from a kernel associated with the ReLU function, and has been studied in~\citep{xie2017diverse,tsuchida2017invariance,du2018provably}.

In our setting of training a two-layer ReLU network, \citet{du2018provably} showed that if $\mat H^\infty$ is positive definite, GD converges to $0$ training loss if $m$ is sufficiently large:
\begin{thm}[\citep{du2018provably}\footnote{\citet{du2018provably} only considered the case $\kappa=1$, but it is straightforward to generalize their result to general $\kappa$ at the price of an extra $1/\kappa^2$ factor in $m$.}]
	\label{thm:ssdu-converge}
	Assume $\lambda_0 = \lambda_{\min}(\mat H^\infty) >0$. For $\delta\in(0, 1)$, if $m = \Omega\left( \frac{n^6}{\lambda_0^4 \kappa^2 \delta^3 } \right)$ and $\eta = O\left( \frac{\lambda_0}{n^2} \right)$, then with probability at least $1-\delta$ over the random initialization~\eqref{eqn:random-init}, we have:
	\begin{itemize}
		\item $\Phi(\mat W(0)) = O(n/\delta)$;
		\item $\Phi(\mat W(k+1)) \le \left( 1 - \frac{\eta \lambda_0}{2}\right)	\Phi(\mat W(k)),\  \forall k\ge0$.
	\end{itemize}
\end{thm}

Our results on optimization and generalization also crucially depend on this matrix $\mat H^\infty$.

\subsection{Overview of Our Results}

Now we give an informal description of our main results.
It assumes that the initialization magnitude $\kappa$ is sufficiently small and the network width $m$ is sufficiently large (to be quantified later).

The following theorem gives a precise characterization of how the objective decreases to $0$.
It says that this process is essentially determined by a power method for matrix $\mat I-\eta\mat H^\infty$ applied on the label vector $\vect y$.
\begin{thm}[Informal version of Theorem~\ref{thm:convergence_rate}] \label{thm:rate-informal}
	With high probability we have:
	\[
	\Phi(\mat W(k)) \approx \frac12 \norm{(\mat I - \eta \mat H^\infty)^k \vect y}_2^2, \quad \forall k\ge 0.
	\]
\end{thm}
As a consequence, we are able to distinguish the convergence rates for different labels $\vect y$, which can be determined by the projections of $\vect y$ on the eigenvectors of $\mat H^\infty$.
This allows us to obtain an answer to Question~\ref{ques:conv}. See Section~\ref{sec:rate} for details.

Our main result for generalization is the following:
\begin{thm}[Informal version of Theorem~\ref{thm:main_generalization}] \label{thm:generalization-informal}
	For any $1$-Lipschitz loss function, the generalization error of the two-layer ReLU network found by GD is at most
	\begin{equation} \label{eqn:gen_error}
		\sqrt{\frac{2 \vect y^\top (\mat H^\infty)^{-1} \vect y}{n}}.
	\end{equation}
\end{thm}
Notice that our generalization bound~\eqref{eqn:gen_error} can be computed from data $\{(\vect x_i, y_i)\}_{i=1}^n$, and is completely independent of the network width $m$.
We observe that this bound can clearly distinguish true labels and random labels, thus providing an answer to Question~\ref{ques:gen}. See Section~\ref{sec:generalization} for details.

Finally, using Theorem~\ref{thm:generalization-informal}, we prove that we can use our two-layer ReLU network trained by GD to learn a broad class of functions, including linear functions, two-layer neural networks with polynomial activation $\phi(z) = z^{2l}$ or cosine activation, etc.
See Section~\ref{sec:improper} for details.

\subsection{Additional Notation}
We introduce some additional notation that will be used.

Define $u_i = f_{\mat{W},\vect{a}}(\vect{x}_i)$, i.e., the network's prediction on the $i$-th input.
We also use $\vect{u} = \left({u}_1,\ldots,{u}_n\right)^\top \in \R^n$ to denote all $n$ predictions.
Then we have $\Phi(\mat W) = \frac12 \norm{\vect y - \vect u}_2^2$ and
the gradient of $\Phi$ can be written as:
\begin{equation} \label{eqn:gradient-w_r}
\frac{\partial \Phi(\mat W)}{\partial \vect{w}_r} = \frac{1}{\sqrt m} a_r \sum_{i=1}^n (u_i - y_i) \indict_{r, i} \vect{x}_i, \quad \forall r\in[m],
\end{equation}
where $ \indict_{r, i} = \indict \left\{ \vect{w}_r^\top \vect x_i \ge 0 \right\} $.

We define two matrices $\mat Z$ and $\mat H$ which will play a key role in our analysis of the GD trajectory:
	\begin{align*}
	\mat Z = \frac{1}{\sqrt m} \begin{pmatrix}
	\indict_{1, 1} a_1 \vect{x}_1 & 
%	\indict_{1, 2} a_1 \vect{x}_2 & 
	\cdots & \indict_{1, n} a_1 \vect{x}_n \\
%	\indict_{2, 1} a_2 \vect{x}_1 & \indict_{2, 2} a_2 \vect{x}_2 & \cdots & \indict_{2, n} a_2 \vect{x}_n \\
	\vdots & 
%	\vdots &
	\ddots & \vdots \\
	\indict_{m, 1} a_m \vect{x}_1 & 
%	\indict_{m, 2} a_m \vect{x}_2 & 
	\cdots & \indict_{m, n} a_m \vect{x}_n
	\end{pmatrix}
		\in \R^{md \times n},
	\end{align*}
	and
	$
		\mat{H} = \mat{Z}^\top \mat{Z}
	$.
	Note that $$\mat{H}_{ij} = \frac{\vect{x}_i^\top \vect{x}_j}{m} \sum_{r=1}^{m}\indict_{r,i}\indict_{r,j}, \quad \forall i,j\in[n].$$
With this notation we have a more compact form of the gradient~\eqref{eqn:gradient-w_r}:
\begin{equation*}
\vectorize{\nabla \Phi(\mat W)} =  \mat Z (\vect u - \vect y).
\end{equation*}
Then the GD update rule is:
\begin{equation} \label{eqn:gd}
\vectorize{\mat W(k+1)} = \vectorize{\mat W(k)} - \eta \mat Z(k) (\vect u(k) - y),
\end{equation}
for  $k=0, 1, \ldots$.
Throughout the paper, we use $k$ as the iteration number,
% and the letter $t$ to denote the continuous time index which we will use in our proof.
and also use $k$ to index all variables that depend on $\mat{W}(k)$.
For example, we have $u_i(k) = f_{\mat{W}(k),\vect{a}}(\vect{x}_i)$, $\indict_{r, i}(k) = \indict\left\{ \vect{w}_r(k)^\top \vect x_i \ge 0 \right\}$, etc.

\section{Analysis of Convergence Rate}
\label{sec:rate}
\begin{figure*}[t]
	\centering
	\subfigure[Convergence Rate, MNIST.]
	{
		\label{fig:mnist_speed}
		\includegraphics[width=0.4\textwidth]{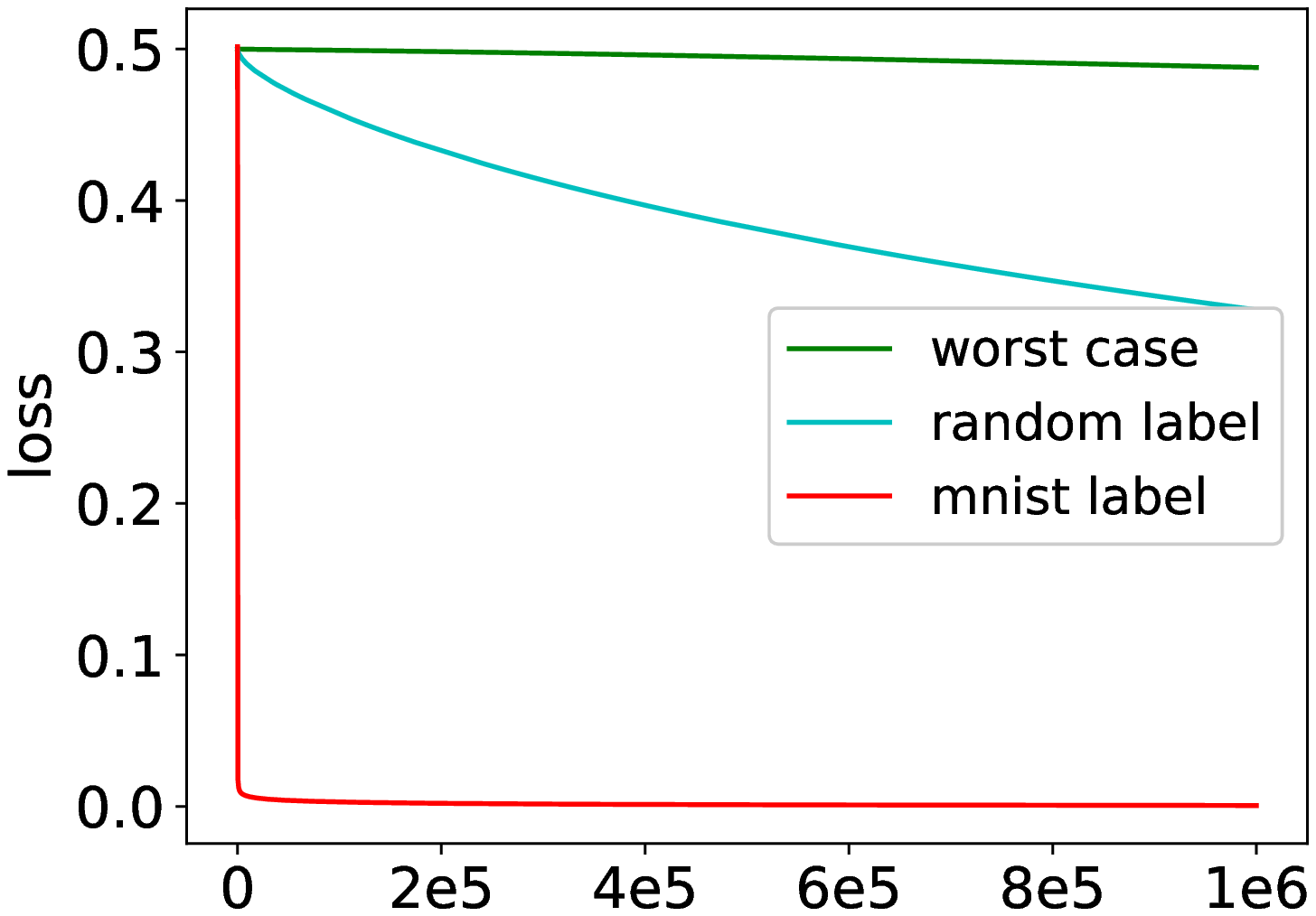}}   
	\subfigure[Eigenvalues \& Projections, MNIST.]
	{
		\label{fig:mnist_spectrum}
		\includegraphics[width=0.44\textwidth]{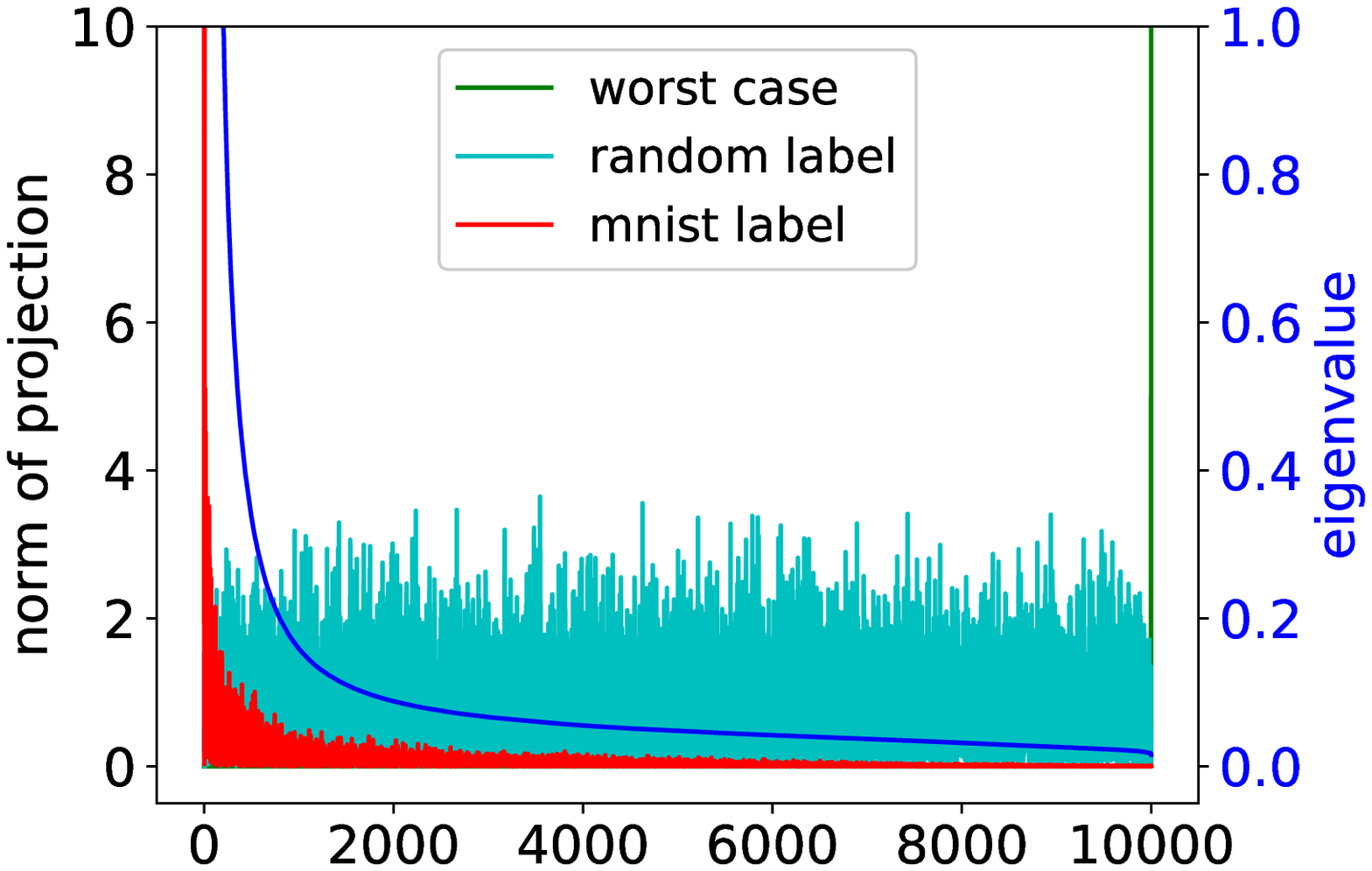}}   \\
	\subfigure[Convergence Rate, CIFAR.]
	{
		\label{fig:cifar_speed}
		\includegraphics[width=0.4\textwidth]{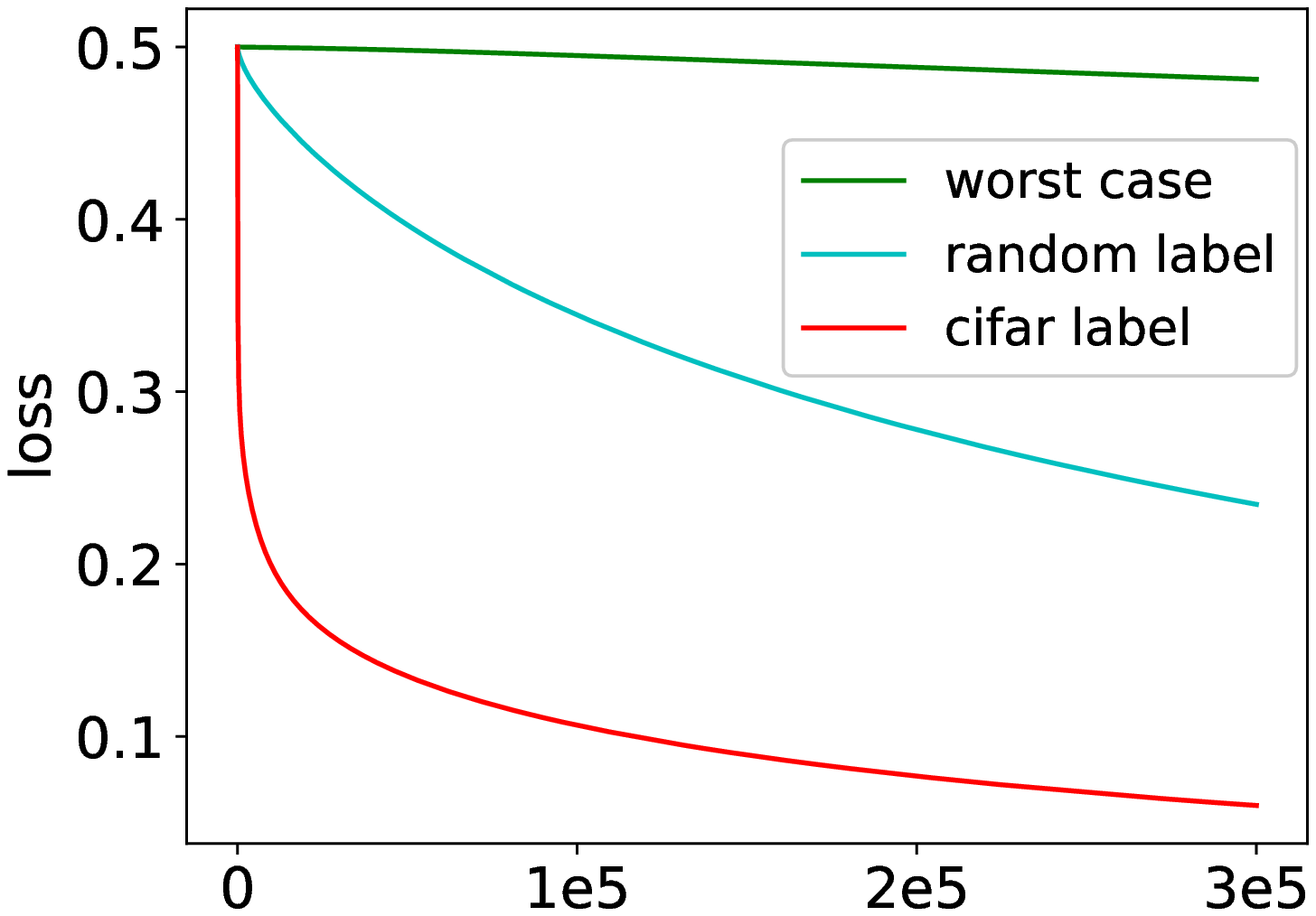}}   
	\subfigure[Eigenvalues \& Projections, CIFAR.]
	{
		\label{fig:cifar_spectrum}
		\includegraphics[width=0.44\textwidth]{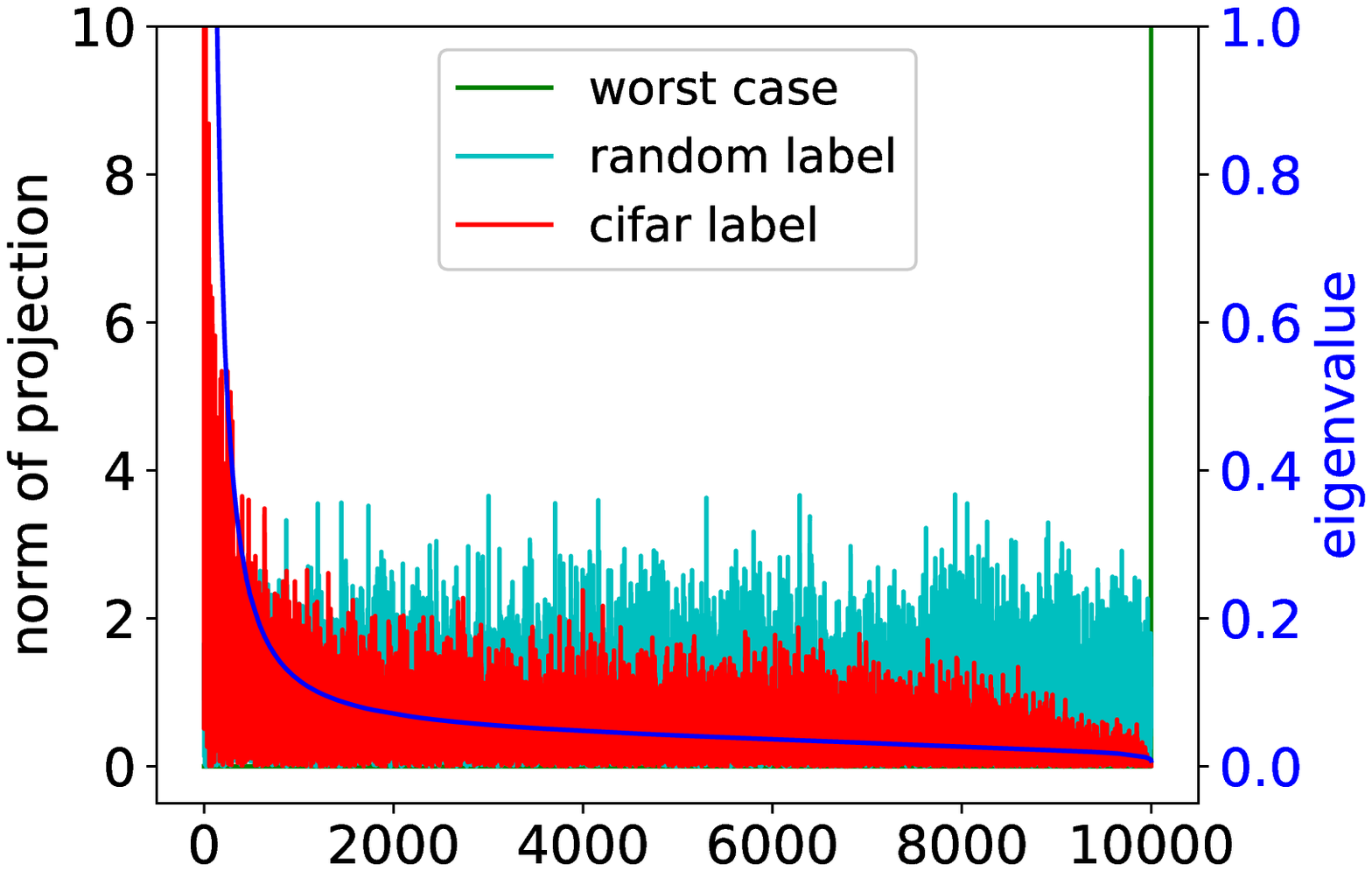}}   
	\caption{In Figures~\ref{fig:mnist_speed} and \ref{fig:cifar_speed}, we compare  convergence rates of gradient descent between using true labels, random labels and the worst case labels (normalized eigenvector of $\vect{H}^\infty$ corresponding to $\lambda_{\min}(\mat{H}^{\infty})$.
		In Figures \ref{fig:mnist_spectrum} and \ref{fig:cifar_spectrum}, we plot the eigenvalues of $\mat H^\infty$ as well as projections of true, random, and worst case labels on different eigenvectors of $\vect{H}^{\infty}$. 
		The experiments use gradient descent on data from two classes of MNIST or CIFAR.
		%				, and the label of each data is set to $1$ or $-1$ according to each sample's category.
		The plots clearly demonstrate that true labels have much better alignment with top eigenvectors, thus enjoying faster convergence.
	}
	\label{fig:rate_mnist}
\end{figure*}

Although Theorem~\ref{thm:ssdu-converge} already predicts linear convergence of GD to $0$ loss, it only provides an upper bound on the loss and does not distinguish different types of labels. In particular, it cannot answer Question~\ref{ques:conv}.
In this section we give a fine-grained analysis of the convergence rate.

Recall the loss function $\Phi(\mat W) = \frac12 \norm{\vect y - \vect u}_2^2$.
%\begin{align*}
%L(\mat{W}) = \frac{1}{2}\sum_{i=1}^{n}\left(y_i-f(\vect{x}_i,\mat{W},\vect{a})\right)^2 = \frac{1}{2}\norm{\vect{y}-\vect{u}}_2^2.
%\end{align*}
Thus, %to study the convergence rate of gradient descent, we need to study how fast the sequence 
it is equivalent to study how fast the sequence $\left\{\vect{u}(k) \right\}_{k=0}^\infty$ converges to $\vect y$.
%$\left\{\norm{\vect{y}-\vect{u}(k)}_2^2\right\}_{k=0}^\infty$ converges to $0$.
Key to our analysis is the observation that when the size of initialization $\kappa$ is small and the network width $m$ is large, the sequence $\left\{\vect{u}(k) \right\}_{k=0}^\infty$ stays close to another sequence $\left\{\tilde{\vect{u}}(k)\right\}_{k=0}^\infty$ which has a \emph{linear} update rule:
\begin{equation}
\begin{aligned}
\tilde{\vect{u}}(0) &= \vect{0}, \\
\tilde{\vect{u}}(k+1) &= \tilde{\vect{u}}(k) - \eta \mat{H}^\infty\left(\tilde{\vect{u}}(k)-\vect{y}\right), \label{eqn:u_hat_dynamics}
\end{aligned}
\end{equation}
where $\mat H^\infty$ is the Gram matrix defined in~\eqref{eqn:H_infy_defn}.

Write the eigen-decomposition $\mat H^\infty = \sum_{i=1}^n \lambda_i \vect v_i \vect v_i^\top$, where $\vect v_1, \ldots, \vect v_n \in \R^n$ are orthonormal eigenvectors of $\mat H^\infty$ and $\lambda_1, \ldots, \lambda_n$ are corresponding eigenvalues.
Our main theorem in this section is the following:
%\begin{thm}\label{thm:convergence_rate}
%With at least $1-\delta$ probability over initialization, we have
%	\begin{align*}
%	\norm{\vect{y}-\vect{u}(k)}_2^2 
%	= &\norm{\vect{y}-\tilde{\vect{u}}(k)}_2^2 + \uerror(k)\\
%	= &\sum_{i=1}^{n}(1-\eta\lambda_i)^k \left(\vect{v}_i^\top \left(\vect{u}(0)-\vect{y}\right)\right)^2 + \uerror(k)
%	\end{align*}
%	where $\lambda_1 \ge \lambda_2 \ge \cdots \lambda_n \ge 0$ are eigenvalues of $\mat{H}^\infty$ and $\vect{v}_1,\vect{v}_2,\ldots,\vect{v}_n$ are corresponding eigenvectors and \begin{align*}
%		\abs{\uerror(k)} = O\left(k\left(1-\frac{\eta \lambda_0}{4}\right)^{k-1}\cdot\frac{\eta n^{5/2}\norm{\vect{y}-\vect{u}(0)}}{\sqrt{m}\delta\lambda_0}\right).
%	\end{align*}
%\end{thm}
\begin{thm}\label{thm:convergence_rate}
%	If $\initvar = \poly(\frac{1}{n},\delta,\lambda_0)$, then 
Suppose $\lambda_0 = \lambda_{\min}(\mat H^\infty) >0$,
$\kappa = O\left( \frac{\epsilon \delta}{\sqrt n} \right)$,
 $m = \Omega\left( \frac{n^7}{\lambda_0^4 \kappa^2 \delta^4 \epsilon^2} \right)$ and $\eta = O\left( \frac{\lambda_0}{n^2} \right)$. 
%Under the same assumptions as in Theorem~\ref{thm:ssdu-converge}, 
Then with probability at least $1-\delta$ over the random initialization, for all $k=0, 1, 2, \ldots$ we have: 
\begin{equation} \label{eqn:u(k)-y_size}
\norm{\vect{y}-\vect{u}(k)}_2
= \sqrt{\sum_{i=1}^{n}(1-\eta\lambda_i)^{2k} \left(\vect{v}_i^\top \vect{y}\right)^2} \pm \epsilon.
\end{equation}
%	\begin{equation} \label{eqn:u(k)-y_size}
%	\norm{\vect{y}-\vect{u}(k)}_2
%	= \sqrt{\sum_{i=1}^{n}(1-\eta\lambda_i)^{2k} \left(\vect{v}_i^\top \vect{y}\right)^2} + \uerror(k),
%	\end{equation}
%	where $\uerror(k) \in \R$ satisfies
%	\begin{equation} \label{eqn:u-error}
%	\begin{aligned}
%	%\abs{\uerror(k)} = O\left(k\left(1-\frac{\eta \lambda_0}{4}\right)^{k-1} \left( \frac{\sqrt n \kappa}{ \delta}  + \frac{\eta n^{7/2}}{\sqrt m \lambda_0 \kappa \delta^{2}} \right) \right).
%	\abs{\uerror(k)} =\ & O\bigg(  (1-\eta \lambda_0)^k \cdot \frac{\sqrt n \kappa}{ \delta}  \\&+ k\left(1-\frac{\eta \lambda_0}{4}\right)^{k-1}\cdot \frac{\eta n^{7/2}}{\sqrt m \lambda_0 \kappa \delta^{2}}  \bigg).
%	\end{aligned}
%	\end{equation}
\end{thm}
%\wei{Will be updated after proof is written down}
%\simon{@wei: Can you verify?}

The proof of Theorem~\ref{thm:convergence_rate} is given in Appendix~\ref{app:proof_rate}.

%Now we discuss the consequence of Theorem~\ref{thm:convergence_rate}.
%Note that $\kappa$ is sufficiently small and $m$ is sufficiently large, the second term $\uerror(k)$ in~\eqref{eqn:u(k)-y_size} is very small for all $k$.
%In particular, when $\kappa \ll \delta/\sqrt n$, the first term in~\eqref{eqn:u-error} is small for all $k$;
%when $m\gg n^7/ \left( \lambda_0^4 \kappa^2 \delta^4 \right)$, the second term in~\eqref{eqn:u-error} is small for all $k$.\footnote{Note that $\max\limits_{k\ge 0} \left\{ k (1-\eta\lambda_0/4)^{k-1} \right\} = O(1/(\eta\lambda_0))$.}
%Hence, %the determining term for convergence is the first term $\sum_{i=1}^{n}(1-\eta\lambda_i)^k \left(\vect{v}_i^\top \vect{y}\right)^2$.
%we essentially have
%\begin{equation} \label{eqn:u(k)-y_approx}
%\norm{\vect{y}-\vect{u}(k)}_2
%\approx \sqrt{ \sum_{i=1}^{n}(1-\eta\lambda_i)^{2k} \left(\vect{v}_i^\top \vect{y}\right)^2 }.
%\end{equation}
In fact, %the right hand side of~\eqref{eqn:u(k)-y_approx} 
the dominating term $\sqrt{\sum_{i=1}^{n}(1-\eta\lambda_i)^{2k} \left(\vect{v}_i^\top \vect{y}\right)^2}$ is exactly equal to $\norm{\vect{y}-\tilde{\vect{u}}(k)}_2$, which we prove in Section~\ref{sec:proof_sketch_rate}.

In light of~\eqref{eqn:u(k)-y_size}, it suffices to understand how fast $\sum_{i=1}^{n}(1-\eta\lambda_i)^{2k} \left(\vect{v}_i^\top \vect{y}\right)^2$ converges to $0$ as $k$ grows.
Define $\useq_i(k) =(1-\eta\lambda_i)^{2k}(\vect{v}_i^\top \vect{y})^2$, and notice that each sequence $\{\useq_i(k)\}_{k=0}^\infty$ is a geometric sequence which starts at $\useq_i(0) =(\vect{v}_i^\top \vect{y})^2$ and decreases at ratio $(1-\eta\lambda_i)^2$.
In other words, we can think of decomposing the label vector $\vect y$ into its projections onto all eigenvectors $\vect v_i$ of $\mat H^\infty$: $\norm{\vect y}_2^2 = \sum_{i=1}^n (\vect{v}_i^\top \vect{y})^2 = \sum_{i=1}^n \useq_i(0)$, and the $i$-th portion shrinks exponentially at ratio $(1-\eta\lambda_i)^2$.
The larger $\lambda_i$ is, the faster $\{\useq_i(k)\}_{k=0}^\infty$ decreases to $0$, so in order to have faster convergence we would like the projections of $\vect y$ onto top eigenvectors to be larger.
Therefore we obtain the following intuitive rule to compare the convergence rates on two sets of labels in a qualitative manner (for fixed $\norm{\vect y}_2$):
\begin{itemize}
	\item For a set of labels $\vect{y}$, if they align with the top eigenvectors, i.e., $(\vect{v}_i^\top \vect{y})^2$ is large for large $\lambda_i$, then gradient descent converges quickly.
	
	\item For a set of labels $\vect{y}$, if the projections on eigenvectors $\{\left(\vect{v}_i^\top \vect{y}\right)^2\}_{i=1}^n$ are uniform, or labels align with eigenvectors with respect to small eigenvalues, then gradient descent converges with a slow rate.
\end{itemize}

%Notice $\useq(k) = \sum_{i=1}^{n}\useq_i(k)$.
% Furthermore, for each sequence $\{\useq_i(k)\}_{k=0}^\infty$, it is not related to any other sequence $\{\useq_j(k)\}_{k=0}^\infty$ for $j \neq i$, so we can analyze each sequence separately.
% The speed of $\{\useq_i(k)\}_{k=0}^\infty$ converging to $0$ depends on two factors.
% The first factor is $(1-\eta \lambda_i)$.
% The larger $\lambda_i$ is, the faster $\{\useq_i(k)\}_{k=0}^\infty$ converging to $0$.
% The other factor is $(\vect{v}_i^\top \vect{y})^2$ which is the magnitude of labels projecting on the direction $\vect{v}_i$.

\paragraph{Answer to Question~\ref{ques:conv}.}
%\paragraph{Explanation of Observation~\ref{obs:conv}.}
We now use this reasoning to answer Question~\ref{ques:conv}.
In Figure~\ref{fig:mnist_spectrum}, we compute the eigenvalues of $\mat{H}^\infty$ (blue curve) for the MNIST dataset.
The plot shows the eigenvalues of $\mat{H}^\infty$ admit a fast decay.
We further compute the projections $\{\abs{\vect{v}_i^\top \vect{y}}\}_{i=1}^n$ of true labels (red) and random labels (cyan).
We observe that there is a significant difference between the projections of true labels and random labels:
true labels align well with top eigenvectors whereas projections of random labels are close to being uniform.
Furthermore, according to our theory, if a set of labels align with the eigenvector associated with the least eigenvalue, the convergence rate of gradient descent will be extremely slow.
We construct such labels and in Figure~\ref{fig:mnist_speed} we indeed observe slow convergence.
%To further verify our theory, we construct a set of adversarial labels which align with the bottom eigenvectors.
%
We repeat the same experiments on CIFAR and have similar observations (Figures~\ref{fig:cifar_speed} and \ref{fig:cifar_spectrum}).
These empirical findings support our theory on the convergence rate of gradient descent.
See Appendix~\ref{sec:expdetails} for implementation details.

%\%\onecolumn

\begin{comment}
\begin{figure}[t]
	\subfigure[Convergence Rate.]
	{
		\label{fig:cifar_speed}
		\includegraphics[width=0.235\textwidth]{figs/cifar_optimization.png}}   
	\subfigure[Eigenvalue and Spectrum.]
	{
		\label{fig:cifar_spectrum}
		\includegraphics[width=0.235\textwidth]{figs/cifar_spectrum_proj.png}}   
	\caption{[Place holder] Comparison of convergence rates of gradient descent using true labels and original labels. We pick two classes, and one class with label $1$ and one class with label $-1$\simon{add some justification of using this labeling scheme for classification.}
		\label{fig:rate_cifar}}
\end{figure}
\end{comment}

%\simon{example1 polynomial decay, example 2: uniform?}

\subsection{Proof Sketch of Theorem~\ref{thm:convergence_rate}}
\label{sec:proof_sketch_rate}

Now we prove $\norm{\vect{y}-\tilde{\vect{u}}(k)}_2^2 = \sum_{i=1}^{n}(1-\eta\lambda_i)^{2k} \left(\vect{v}_i^\top \vect{y}\right)^2$.
The entire proof of Theorem~\ref{thm:convergence_rate} is given in Appendix~\ref{app:proof_rate}, which relies on the fact that the dynamics of $\{\vect u(k)\}_{k=0}^\infty$ is essentially a perturbed version of~\eqref{eqn:u_hat_dynamics}.

From~\eqref{eqn:u_hat_dynamics} we have $\tilde{\vect{u}}(k+1) - \vect y = (\mat I - \eta \mat{H}^\infty) \left(\tilde{\vect{u}}(k)-\vect{y}\right)$, which implies $\tilde{\vect{u}}(k)-\vect{y} = (\mat I - \eta \mat{H}^\infty)^k \left(\tilde{\vect{u}}(0)-\vect{y}\right) = - (\mat I - \eta \mat{H}^\infty)^k \vect{y}$.
Note that $(\mat I - \eta \mat{H}^\infty)^k$ has eigen-decomposition $(\mat I - \eta \mat{H}^\infty)^k = \sum_{i=1}^n(1-\eta\lambda_i)^k \vect v_i \vect v_i^\top$ and that $\vect y$ can be decomposed as $\vect y = \sum_{i=1}^n (\vect v_i^\top \vect y) \vect v_i$.
Then we have $\tilde{\vect{u}}(k)-\vect{y} = - \sum_{i=1}^n(1-\eta\lambda_i)^k (\vect v_i^\top \vect y) \vect v_i$, which implies $\norm{\tilde{\vect{u}}(k)-\vect{y}}_2^2 =  \sum_{i=1}^n (1-\eta\lambda_i)^{2k} (\vect v_i^\top \vect y)^2$.

\section{Analysis of Generalization}
\label{sec:generalization}

In this section, we study the generalization ability of the two-layer neural network $f_{\mat W(k), \vect a}$ trained by GD.

First, in order for optimization to succeed, i.e., zero training loss is achieved,
we need a \emph{non-degeneracy} assumption on the data distribution, defined below:
\begin{defn} \label{def:non-degenerate distribution}
	A distribution $\calD$ over $\R^d \times \R$ is $(\lambda_0, \delta, n)$-non-degenerate, if for $n$ i.i.d. samples $\{(\vect x_i, y_i)\}_{i=1}^n$ from $\calD$, with probability at least $1-\delta$ we have $\lambda_{\min}(\mat H^\infty) \ge \lambda_0 >0$.
\end{defn}
\begin{rem}
	Note that as long as no two $\vect x_i$ and $\vect x_j$ are parallel to each other, we have $\lambda_{\min}(\mat H^\infty) > 0$. (See~\citep{du2018provably}).
	For most real-world distributions, any two training inputs are not parallel.
\end{rem}

Our main theorem is the following:
\begin{thm}\label{thm:main_generalization}
	Fix % an error parameter $\epsilon>0$ and 
	a failure probability $\delta \in (0, 1)$.
	Suppose our data $S = \left\{(\vect x_i,  y_i)\right\}_{i=1}^n$ are i.i.d. samples from a $(\lambda_0, \delta/3, n)$-non-degenerate distribution $\calD$, and $\kappa = O\left( \frac{ \lambda_0 \delta}{ n} \right), m\ge \kappa^{-2} \poly\left(n, \lambda_0^{-1}, \delta^{-1} \right)  $.
	Consider any loss function $\ell: \R\times\R \to [0, 1]$ that is $1$-Lipschitz in the first argument such that $\ell(y, y)=0$.
	Then with probability at least $1-\delta$ over the random initialization and the training samples, the two-layer neural network $f_{\mat W(k), \vect a}$ trained by GD for $k\ge \Omega\left( \frac{1}{\eta\lambda_0} \log\frac{n}{\delta} \right)$ iterations has population loss $L_\calD(f_{\mat W(k), \vect a}) = \E_{(\vect x, y)\sim\calD}\left[ \ell(f_{\mat W(k), \vect a}(\vect x), y) \right]$ bounded as:
	\begin{equation} \label{eqn:main_generalization}
	L_\calD(f_{\mat W(k), \vect a}) \le \sqrt{\frac{2\vect{y}^\top \left(\mat{H}^{\infty}\right)^{-1}\vect{y}}{n}}  + O\left( \sqrt{\frac{\log\frac{n}{\lambda_0\delta}}{n}} \right) .
	%\sqrt{\frac{2 \vect y^\top (\mat H^\infty)^{-1}\vect y }{n}} + 3 \sqrt{\frac{\log(6/\delta)}{2n}} + \epsilon.
	\end{equation}
	%Suppose $\mu$ satisfies that there exists a function $\lambda_{\mu}: \mathbb{R}\rightarrow \mathbb{R}$ such that with probability at least $1-\delta_1$, $\lambda_{\min}(\mat{H}^\infty) \ge \lambda_{\mu}(\delta_1)$.\simon{better way of stating this assumption?}
%	We let $L_{tr}(k) = \frac{1}{n}\sum_{i=1}^{n}\ell(f(\mat{W}(k),\vect{a},\vect{x}_i),y_i)$ and $L_{te}(k) = \expect_{(\vect{x},y) \sim \mu }\left[\ell(f(\mat{W}(k),\vect{a},\vect{x}),y)\right]$.
%	Then if we set $m = \poly(n,\epsilon,1/\lambda_{\mu}(\delta_1),1/\delta_2)$ we have with failure probability at most $(\delta+\delta_1)$ over the training samples and with failure probability at most  $\delta_2$ over the random initialization, for $k=0,1,\ldots$\begin{align}
%	&L_{te}(k)-L_{tr}(k) \nonumber\\
%	\le &2\rho\sqrt{\frac{ \left(\vect{y}^\top \left(\mat{H}^\infty\right)^{-1}\vect{y}\right) \tr\left(\mat{H}^{\infty}\right)+\log (\delta^{-1})}{n}}+ \epsilon . \label{eqn:main_generalization}
%	\end{align}
\end{thm}

The proof of Theorem~\ref{thm:main_generalization} is given in Appendix~\ref{app:proof_generalization} and we sketch the proof in Section~\ref{sec:proof_sketch_generalization}.

Note that in Theorem~\ref{thm:main_generalization} there are three sources of possible failures: (i) failure of satisfying $\lambda_{\min}(\mat H^\infty) \ge \lambda_0$, (ii) failure of random initialization, and (iii) failure in the data sampling procedure (c.f. Theorem~\ref{thm:rad_generalization}). We ensure that all these failure probabilities are at most $\delta/3$ so that the final failure probability is at most $\delta$.

%$\delta$ represents the failure probability of the sampling procedure (c.f. Theorem~\ref{thm:rad_generalization}).
%$\lambda_{\mu}(\cdot)$ is a function that depends on the distribution $\mu$ and characterizes the Gram matrix least eigenvalue given a tolerance failure probability $\delta_1$.
%In Appendix we give justification on the existence of $\lambda_{\mu}(\cdot)$.
%Lastly, we use $\delta_2$ to characterize the failure probability from random initialization.
%Similar to $\epsilon$, both $\delta_1$ and $\delta_2$ can be arbitrarily small if we set $m$ large enough.

As a corollary of Theorem~\ref{thm:main_generalization}, for binary classification problems (i.e., labels are $\pm1$), we can show that \eqref{eqn:main_generalization} also bounds the \emph{population classification error} of the learned classifier. 
See Appendix~\ref{app:proof_generalization} for the proof.
\begin{cor} \label{cor:binary-classification-generalization}
	Under the same assumptions as in Theorem~\ref{thm:main_generalization} and additionally assuming that $y\in\{\pm1\}$ for $(\vect x, y)\sim\cD$, with probability at least $1-\delta$, the population classification error $L_\calD^{01}(f_{\mat W(k), \vect a}) = \Pr_{(\vect x, y)\sim\calD}\left[ \sgn\left(f_{\mat W(k), \vect a}(\vect x) \right) \not= y \right]$ is bounded as:
	\begin{equation*} 
	L_\calD^{01}(f_{\mat W(k), \vect a})  \le \sqrt{\frac{2\vect{y}^\top \left(\mat{H}^{\infty}\right)^{-1}\vect{y}}{n}}  + O\left( \sqrt{\frac{\log\frac{n}{\lambda_0\delta}}{n}} \right) .
	\end{equation*}
\end{cor}

Now we discuss our generalization bound.
The dominating term in \eqref{eqn:main_generalization} is:
\begin{equation}
\sqrt{\frac{ 2 \vect{y}^\top \left(\mat{H}^\infty\right)^{-1}\vect{y} }{n}}. \label{eqn:complexity}
\end{equation}
This can be viewed as a \emph{complexity measure of data} that one can use to predict the test accuracy of the learned neural network.
Our result has the following advantages:
(i) our complexity measure~\eqref{eqn:complexity} can be directly computed given data $\{(\vect x_i, y_i)\}_{i=1}^n$, without the need of training a neural network or assuming a ground-truth model;
(ii) our bound is completely independent of the network width $m$.
% (iii) our theorem does not require early stopping of optimization as in~\cite{allen2018learning}.

%Importantly, this complexity measure only depends on the labels $\vect{y}$ and the Gram matrix $\mat{H}^\infty$ which is induced by the input data and the network architecture.
%Comparing with existing generalization bound, our theorem is in dependent of number of parameters in contrast to classical VC-dimension theory based theory and our theorem does not depends on the norm of learned matrices as in recent work.
%Furthermore, our theorem does not require early stopping in~\cite{allen2018learning}.

\paragraph{Evaluating our completixy measure~\eqref{eqn:complexity}.}
%\paragraph{Answer to Question~\ref{ques:gen}.}
To illustrate that the complexity measure in~\eqref{eqn:complexity} effectively determines test error, in Figure~\ref{fig:generalization}
%\ref{fig:generalization_mnist} and Figure~\ref{fig:generalization_cifar}
 we compare this complexity measure versus the test error with true labels and random labels (and mixture of true and random labels).
Random and true labels have significantly different complexity measures, and as the portion of random labels increases, our complexity measure also increases.
See Appendix~\ref{sec:expdetails} for implementation details.
%\simon{@ruosong, please add plots here. One for test error and one for complexity measure. For both MNIST and CIFAR. I will edit later once you have the plots.}
\begin{figure}[t]
	\centering
	\subfigure[MNIST Data.]
	{
		\label{fig:generalization_mnist}
		\includegraphics[width=0.4\textwidth]{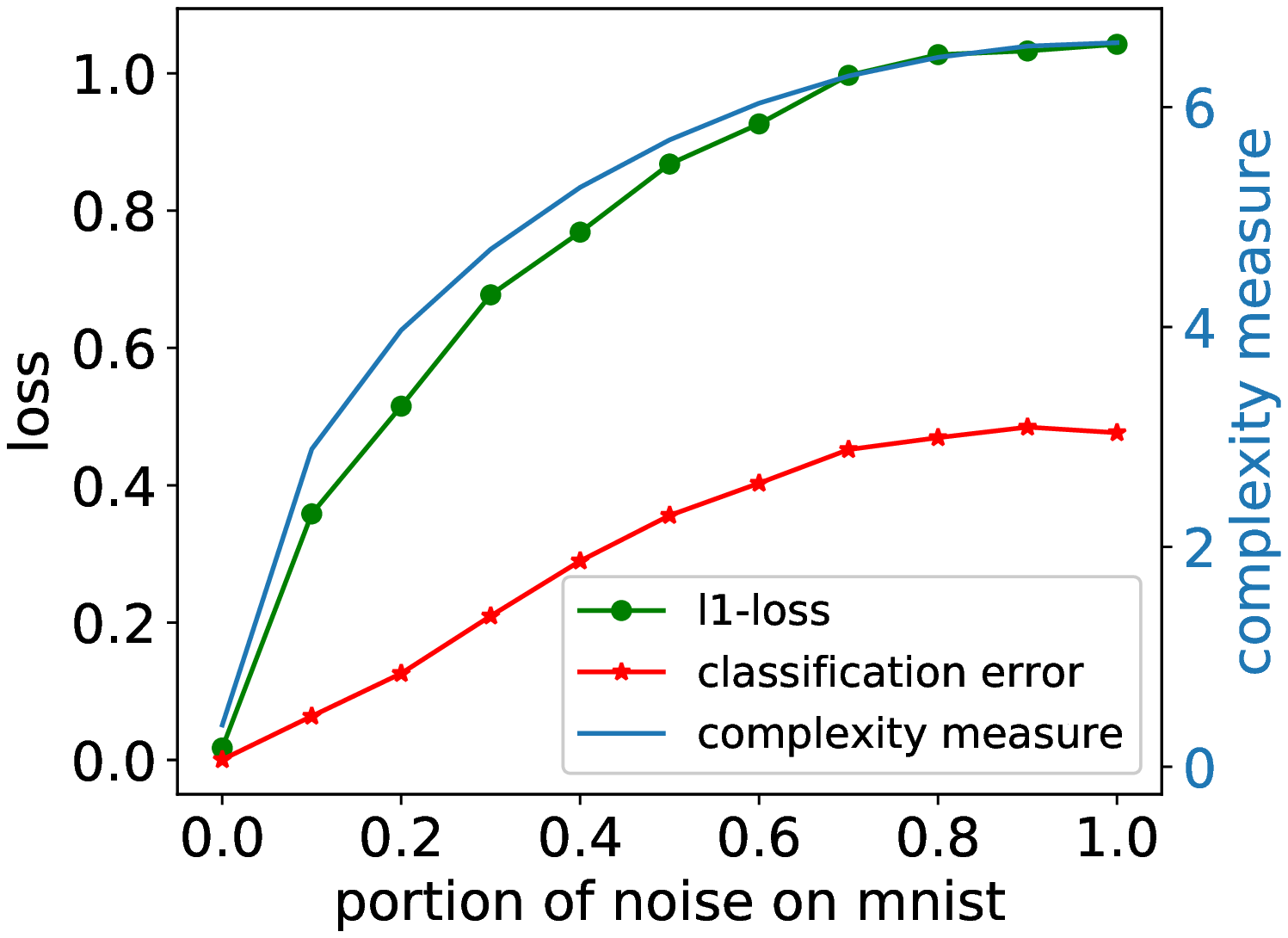}}   
	\subfigure[CIFAR Data.]
	{
		\label{fig:generalization_cifar}
		\includegraphics[width=0.4\textwidth]{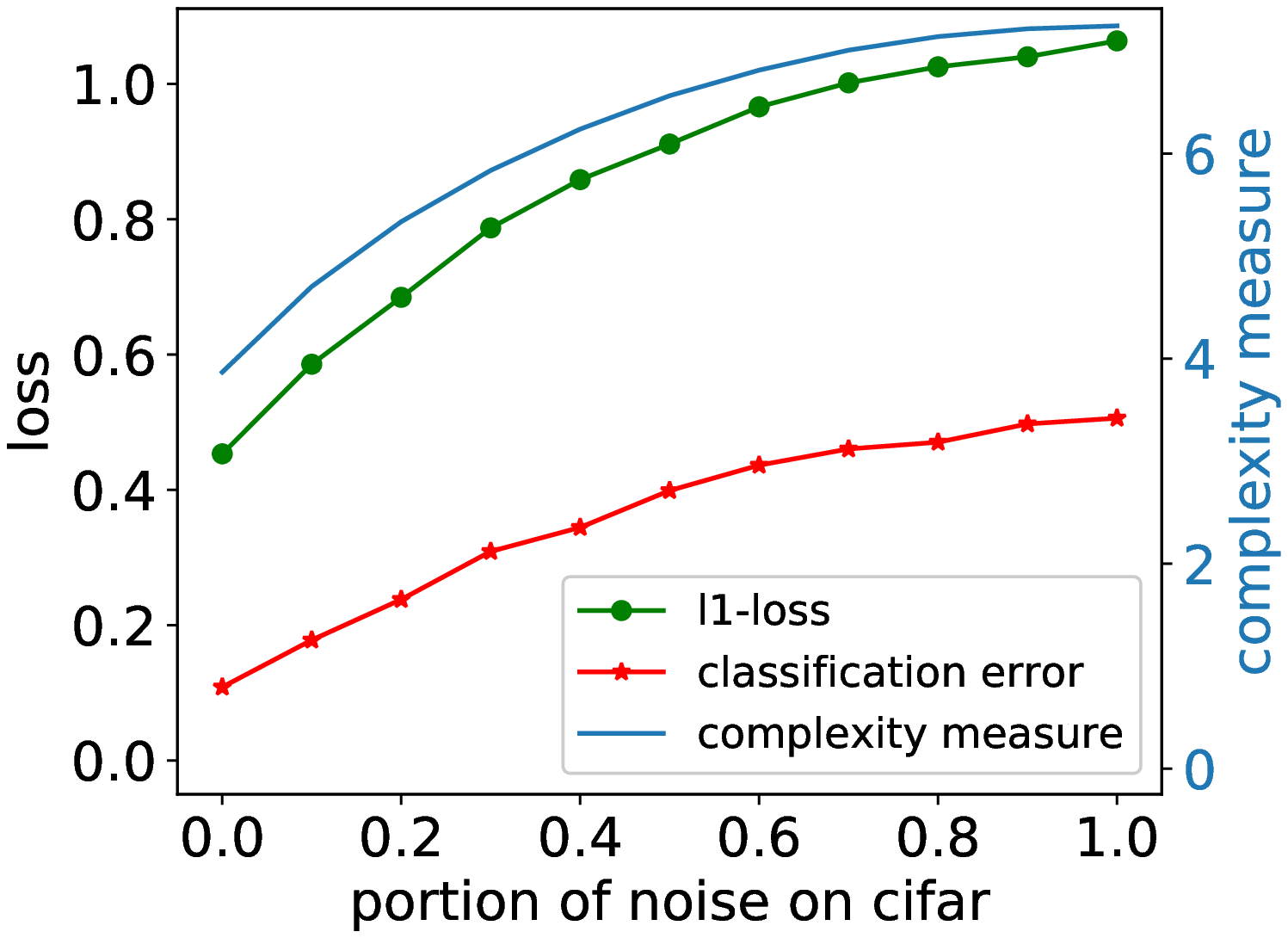}}   
	\caption{Generalization error ($\ell_1$ loss and classification error) v.s. our complexity measure when different portions of random labels are used.
		%Comparison of generalization error using different portion of random labels. 
		We apply GD on data from two classes of MNIST or CIFAR until convergence.
			Our complexity measure almost matches the trend of generalization error as the portion of random labels increases.
			Note that $\ell_1$ loss is always an upper bound on the classification error.
		}
	\label{fig:generalization}
	\vspace{-0.5cm}
\end{figure}

\subsection{Proof Sketch of Theorem~\ref{thm:main_generalization}} \label{sec:proof_sketch_generalization}

The main ingredients in the proof of Theorem~\ref{thm:main_generalization} are Lemmas~\ref{lem:distance_bounds} and \ref{lem:rad_dist_func_class}. We defer the proofs of these lemmas as well as the full proof of Theorem~\ref{thm:main_generalization} to Appendix~\ref{app:proof_generalization}.

Our proof is based on a careful characterization of the trajectory of $\left\{\mat W(k) \right\}_{k=0}^\infty$ during GD.
In particular, we bound its \emph{distance to initialization} as follows:
\begin{lem}\label{lem:distance_bounds}
	Suppose $m \ge \kappa^{-2} \poly\left(n, \lambda_0^{-1}, \delta^{-1} \right)  $ and $\eta = O\left( \frac{\lambda_0}{n^2} \right)$.
	Then with probability at least $1-\delta$ over the random initialization, we have for all $k\ge0$:
	\begin{itemize}
		\item $\norm{\vect w_r(k) - \vect w_r(0)}_2 =O\left( \frac{ n }{\sqrt m \lambda_0\sqrt{\delta}} \right)$ $(\forall r\in[m])$, and
		%$\norm{\mat{W}(k) - \mat{W}(0)}_{2,\infty} \le \frac{\poly(n,1/\lambda_0)}{\sqrt{m}}$,
		\item $\norm{\mat{W}(k)-\mat{W}(0)}_{F} \le \sqrt{\vect{y}^\top \left(\mat{H}^{\infty}\right)^{-1}\vect{y}} + O\left( \frac{  n \kappa}{\lambda_0 \delta} \right) + \frac{\poly\left(n, \lambda_0^{-1}, \delta^{-1} \right)}{m^{1/4} \kappa^{1/2}} $.
		%\item $\norm{\mat{Z}(0)}_F \le \sqrt{\tr\left(\mat{H}^\infty\right)} + \frac{\poly(n,1/\lambda_0)}{\sqrt{m}}$.
	\end{itemize}
\end{lem}
The bound on the movement of each $\vect w_r$ was proved in \cite{du2018provably}. %and the bound on $\mat{Z}(0)$ is a simple application of Hoeffding inequality.
Our main contribution is the bound on $\norm{\mat{W}(k)-\mat{W}(0)}_F$ which corresponds to the \emph{total movement of all neurons}.
The main idea is to couple the trajectory of $\left\{\mat W(k) \right\}_{k=0}^\infty$ with another simpler trajectory $\left\{\widetilde{\mat{W}}(k)\right\}_{k=0}^\infty$ defined as:
%\begin{equation}
\begin{align}
\widetilde{\mat{W}}(0) =\,& \mat{0}, \nonumber\\
\vectorize{\widetilde{\mat{W}}(k+1)} =\,& \vectorize{\widetilde{\mat{W}}(k)} \label{eqn:W_tilde_traj} \\&- \eta\mat{Z}(0)\left( \mat{Z}(0)^\top \vectorize{\widetilde{\mat{W}}(k)}-\vect{y}\right). \nonumber
\end{align}
%\end{equation}
We prove $\norm{\widetilde{\mat W}(\infty) - \widetilde{\mat W}(0)}_F = \sqrt{\vect y^\top \mat H(0)^{-1} \vect y}$ in Section~\ref{sec:proof_sketch_distance}.\footnote{Note that we have $\mat H(0) \approx \mat H^\infty$ from standard concentration. See Lemma~\ref{lem:H(0)-H^inf}.}
The actually proof of Lemma~\ref{lem:distance_bounds} is essentially a perturbed version of this.

Lemma~\ref{lem:distance_bounds} implies that the learned function $f_{\mat W(k), \vect a}$ from GD is in a \emph{restricted} class of neural nets whose weights are close to initialization $\mat W(0)$.
The following lemma bounds the Rademacher complexity of this function class:
%We apply the generic Rademacher complexity method described in Section~\ref{sec:pre} to analyze the generalization ability.
%We consider the following function class.

%This functions in this class are two-layer neural networks with ReLU activation.
%The function class is defined a specific initialization $(\mat{W}(0), \vect{a})$ and two constants $R$ and $B$ that control the complexity of this function class.
%The only variable in this function class is the first layer weight matrix $\mat{W}$.
%Here $R$ controls its distance from $\mat{W}_0$ in $(2,\infty)$ norm and $B$ controls the distance in Frobenius norm.
%The next theorem bounds the Rademacher complexity of this function class.

\begin{lem}\label{lem:rad_dist_func_class}
%	For any fixed $\{\vx_i\}_{i=1}^n$ s.t. $\norm{\vx_i}\le 1$,$\forall i\in [n]$, suppose $\{\vw_r(0)\}_{r=1}^m$,$\{a_r\}_{r=1}^m$ are independently sampled from $N(\bm{0}, \initvar^2I_d)$, $\textrm{Unif}[-1,1]$ respectively. 
Given $R>0$,
with probability at least $1-\delta$ over the random initialization ($\mat W(0),\vect a)$, simultaneously for every $B>0$, the following function class
\begin{align*}
\cF^{\mat{W}(0),\va}_{R,B}  = \{ f_{\mat W, \vect a} : \norm{\vect w_r - \vect w_r(0)}_2 \le R \, (\forall r\in[m]), \\ \norm{\mat W - \mat W(0)}_F \le B \}
%&\cF^{\mat{W}(0),\va}_{R,B} = \{f:\reals^d \to \reals \mid f(\vx) = \sum_{r=1}^m\frac{ a_r}{\sqrt{m}} \relu{\vw_r^\top \vx}, \nonumber\\ &\quad \norm{\mat{W}(0)-\mat{W}}_{2,\infty}\le R,\norm{\vW(0)-\vW}_F\le B \}. \label{eqn:rad_func_class}
\end{align*}
has empirical Rademacher complexity bounded as:
%Then with high probability over the initialization the empirical Rademacher complexity has the following upper bound 
\begin{align*}
	&\calR_S\left( \cF^{\mat W(0),\va}_{R,B} \right)
   =  \frac{1}{n} \expect_{\bm \eps \in \{\pm1\}^n}\left[\sup_{f \in \cF^{\mat W(0),\va}_{R,B}}\sum_{i=1}^{n} \eps_i f(\vect{x}_i)\right]
	\\
	\le \,&	\frac{B}{\sqrt{2n}} \left(1+\left( \frac{2\log \frac2\delta}{m} \right)^{1/4} \right) + \frac{2 R^2 \sqrt m}{ \kappa} + R \sqrt{2\log \frac2\delta} .
\end{align*}
\end{lem}

Finally, combining Lemmas~\ref{lem:distance_bounds} and \ref{lem:rad_dist_func_class}, we are able to conclude that the neural network found by GD belongs to a function class with Rademacher complexity at most $\sqrt{\vect y^\top (\mat H^\infty)^{-1} \vect y /(2n)}$ (plus negligible errors).
This gives us the generalization bound in Theorem~\ref{thm:main_generalization} using the theory of Rademacher complexity (Appendix~\ref{app:rademacher}).

%Lemma~\ref{lem:rad_dist_func_class} is related to 
%inspired by recent works by~\citep{behnam,vashnav} who also used the distance from initialization to characterize the complexity of the function class.
%The key difference here is that we use both $(2,\infty)$ norm and Frobenius norm to control the complexity.
%The proof is thus more involved than previous papers and we refer readers to appendix for the whole proof.

%While Theorem~\ref{thm:rad_dist_func_class} and Rademacher complexity bounds in ~\citep{behnam,vashnav} induce new generalization bounds that depend on the distance from the initialization, these only account for half of the generalization story, i.e., we do not know when the distance is large and when the distance is small.
%Using the idea developed in this paper, by analyzing the trajectory of $\mat{W}(k)$, we are able to complete the story.
%The following theorem gives bounds on $B$, $R$ and $\norm{\vZ(0)}_F$.

%Now combining Theorem~\ref{thm:rad_generalization},~\ref{thm:rad_dist_func_class} and~\ref{thm:distance_bounds}, we obtain our final generalization result in this paper.

\subsection{Analysis of the Auxiliary Sequence $\left\{\widetilde{\mat{W}}(k)\right\}_{k=0}^\infty$}
\label{sec:proof_sketch_distance}
Now we give a proof of  $\norm{\widetilde{\mat W}(\infty) - \widetilde{\mat W}(0)}_F = \sqrt{\vect y^\top \mat H(0)^{-1} \vect y}$ as an illustration for the proof of Lemma~\ref{lem:distance_bounds}.
Define $\vect v(k) = \mat Z(0)^\top \vectorize{\widetilde{\mat W}(k)} \in \R^n$.
Then from~\eqref{eqn:W_tilde_traj} we have $\vect v(0)=\vect0$ and $\vect v(k+1) = \vect v(k) - \eta \mat H(0) (\vect v(k) - \vect y)$, yielding
$\vect v(k) - \vect y = - (\mat I - \eta \mat H(0))^k \vect y$.
Plugging this back to~\eqref{eqn:W_tilde_traj} we get
$\vectorize{\widetilde{\mat{W}}(k+1)} - \vectorize{\widetilde{\mat{W}}(k)} = \eta \mat Z(0) (\mat I - \eta \mat H(0))^k \vect y$.
Then taking a sum over $k=0, 1, \ldots$ we have
\vspace{-0.2cm}
\begin{align*}
\vectorize{\widetilde{\mat{W}}(\infty)} - \vectorize{\widetilde{\mat{W}}(0)} &= \sum_{k=0}^\infty \eta \mat Z(0) (\mat I - \eta \mat H(0))^k \vect y \\
&= \mat Z(0) \mat H(0)^{-1} \vect y.
\end{align*}
\vspace{-0.2cm}
The desired result thus follows:
\begin{align*}
\norm{\widetilde{\mat W}(\infty) - \widetilde{\mat W}(0)}_F^2
&= \vect y^\top \mat H(0)^{-1} \mat Z(0)^\top \mat Z(0) \mat H(0)^{-1} \vect y \\
&= \vect y^\top \mat H(0)^{-1} \vect y.
\end{align*}

\section{Provable Learning using Two-Layer ReLU Neural Networks}
\label{sec:improper}

Theorem~\ref{thm:main_generalization} determines that $\sqrt{\frac{2 \vect y^\top (\mat H^\infty)^{-1}\vect y }{n}}$ controls the generalization error.
% of the two-layer neural networks trained by GD.
In this section, we study what functions can be provably learned in this setting.
We assume the data satisfy $y_i = g(\vect{x}_i)$ for some underlying function $g:\R^d\to\R$.
A simple observation is that if we can prove $$\vect y^\top (\mat H^\infty)^{-1}\vect y \le M_g$$ for some quantity $M_g$ that is \emph{independent} of the number of samples $n$, then Theorem~\ref{thm:main_generalization} implies we can provably learn the function $g$ on the underlying data distribution using $ O\left( \frac{M_g+\log(1/\delta)}{\epsilon^2} \right) $ samples.
The following theorem shows that this is indeed the case for a broad class of functions.

\begin{thm}\label{thm:improper_learning_monomial}
	Suppose we have
	\begin{align*}
	y_i = g(\vect{x}_i) = \alpha \left( \vbeta^\top \vx_i \right)^p, \quad \forall i\in[n],
	\end{align*}
	where $p=1$ or $p=2l$ $(l\in\mathbb N_+)$,	$\vbeta \in \R^d$ and $\alpha\in \R$.
	Then we have
\vspace{-0.3cm}
	\begin{align*}
	\sqrt{\vect y^\top (\mat H^\infty)^{-1}\vect y } \le  3 p \abs{\alpha} \cdot \norm{\vbeta}_2^{p}.
	\end{align*}
\vspace{-0.7cm}
\end{thm}

The proof of Theorem~\ref{thm:improper_learning_monomial} is given in Appendix~\ref{app:proof_improper}.

Notice that for two label vectors $\vect y^{(1)}$ and $\vect y^{(2)}$, we have
\begin{align*}
&\sqrt{(\vy^{(1)}+\vy^{(2)})^\top (\vect{H}^{\infty})^{-1}  \left(\vy^{(1)}+\vy^{(2)}\right)} \\
\le\,&
\sqrt{(\vy^{(1)})^\top (\vect{H}^{\infty})^{-1}  \vy^{(1)}} + \sqrt{(\vy^{(2)})^\top (\vect{H}^{\infty})^{-1}  \vy^{(2)}}.
\end{align*}
This implies that \emph{the sum of learnable functions is also learnable}. Therefore, the following is a direct corollary of Theorem~\ref{thm:improper_learning_monomial}:
\begin{cor}\label{cor:improper_learning_polynomial}
	Suppose we have 
	\begin{align}
	y_i = g(\vect x_i) = \sum_{j} \alpha_j \left( \vbeta_j^\top \vx_i \right)^{p_j} , \quad \forall i\in[n],
	\label{eqn:g_decomp}
	\end{align}
	where for each $j$, $p_j \in \{1, 2, 4, 6, 8, \ldots\}$, $\vbeta_j \in \R^d$ and $\alpha_j \in \R$.
	Then we have
	\begin{align} \label{eqn:g_complexity}
	\sqrt{\vect y^\top (\mat H^\infty)^{-1}\vect y } \le 3 \sum_{j} p_j \abs{\alpha_j} \cdot \norm{\vbeta_j}_2^{p_j}.
	\end{align}
\end{cor}

\vspace{-0.3cm}
Corollary~\ref{cor:improper_learning_polynomial} shows that  overparameterized two-layer ReLU network can learn any function of the form~\eqref{eqn:g_decomp} for which \eqref{eqn:g_complexity} is bounded.
One can view \eqref{eqn:g_decomp} as two-layer neural networks with polynomial activation $\phi(z)=z^p$, where $\{\vbeta_j\}$ are weights in the first layer and $\{\alpha_j\}$ are the second layer.
%Equation~\eqref{eqn:g_complexity} quantifies the smoothness of the underlying function.
%Intuitively, if the decomposition of $g$ has fast decaying Taylor expansion coefficients, then we can learn $g$ with few samples.
Below we give some specific examples.

\begin{example}[Linear functions] \label{example:linear}
	For $g(\vx) = \vbeta^\top \vx$, we have $M_g = O(\norm{\vbeta}_2^2)$.
\end{example}

\begin{example}[Quadratic functions] \label{example:quadratic}
	For $g(\vx) = \vx^\top \mat A \vx$ where $\mat A \in \R^{d\times d}$ is symmetric,
	we can write down the eigen-decomposition $\mat A = \sum_{j=1}^d \alpha_j \vbeta_j \vbeta_j^\top$.
	Then we have $g(\vx) = \sum_{j=1}^d \alpha_j (\vbeta_j^\top \vx)^2$, so $M_g = O\left( \sum_{i=1}^d \abs{\alpha_j} \right) = O(\norm{\mat{A}}_*)$.\footnote{$\norm{\mat{A}}_*$ is the trace-norm of $\mat A$.}
	This is also the class of two-layer neural networks with quadratic activation.
\end{example}

%\begin{example}[Learning Polynomials]\label{example:poly}
%Suppose the underlying function $g$ has the form $g(\vect{x}) = \vbeta_1^\top \vect{x} +\sum_{j=1}^{p} (\vbeta_{2j}^\top \vect{x})^{2j}$ for some small $p \in \mathbb{Z}^+$.
%Theorem~\ref{thm:improper_learning} implies that we can learn this function with $O((\norm{\vbeta_1}+\sum_{j=1}^{p}\norm{\vbeta_{2j}})^2)$ number of samples.
%Note for $p$ small, in the typical scaling $O((\norm{\vbeta_1}+\sum_{j=1}^{p}\norm{\vbeta_{2j}})^2) = O(d)$ because all $\vbeta_j$s are $d$-dimensional vectors.
%Therefore we know we can learn low-degree polynomials with $O(d)$ number of samples, which is the optimal sample complexity.\simon{}
%\end{example}

\begin{example}[Cosine activation]\label{example:cos}
Suppose $g(\vect{x}) = \cos(\vbeta^\top \vect{x})-1$ for some $\vbeta \in \mathbb{R}^d$.
Using Taylor series we know $g(\vect{x}) = \sum_{j=1}^{\infty}\frac{(-1)^j(\vbeta^\top \vect{x})^{2j}}{(2j)!}$.
Thus we have
$M_g = O\left( \sum_{j=1}^\infty \frac{j}{(2j)!} \norm{\vbeta}_2^{2j} \right) = O\left( \norm{\vbeta}_2 \cdot \sinh(\norm{\vbeta}_2) \right) $.
\end{example}

Finally, we note that our ``smoothness'' requirement~\eqref{eqn:g_complexity} is weaker than that in \citep{allen2018learning}, as illustrated in the following example.
\begin{example}[A not-so-smooth function]
	Suppose $g(\vect{x}) = \phi(\vbeta^\top \vx)$, where $\phi(z) = z \cdot \arctan(\frac z2)$ and $\norm{\vbeta}_2 \le 1$.
	We have $g(\vx) =   \sum_{j=1}^\infty \frac{(-1)^{j-1} 2^{1 - 2 j} }{2 j - 1} \left( \vbeta^\top \vx \right)^{2j}$ since $\abs{\vbeta^\top \vx} \le 1$.
	Thus $M_g = O \left( \sum_{j=1}^\infty \frac{j \cdot 2^{1 - 2 j} }{2 j - 1}  \norm{\vbeta}_2^{2j} \right) \le O \left( \sum_{j=1}^\infty 2^{1 - 2 j}  \norm{\vbeta}^{2j} \right) = O\left( \norm{\vbeta}_2^2 \right)$, so our result implies that this function is learnable by $2$-layer ReLU nets.

\end{example}
However, \citet{allen2018learning}'s generalization theorem would require $\sum_{j=1}^\infty \frac{\left(  C\sqrt{\log(1/\epsilon)} \right)^{2j} 2^{1 - 2 j} }{2 j - 1}  $ to be bounded, where $C$ is a large constant and $\epsilon$ is the target generalization error. This is clearly not satisfied.

%\simon{other example?}

%\begin{example}[Learning ReLU Function]\label{example:relu}
%Unfortunately, our theorem also demonstrate it is hard to learn a ReLU function of the form $g(\vect{x},\vbeta) = \relu{\vbeta^\top \vect{x}}$.
%The main reason is that the coefficients of $\relu{\cdot}$ is large\simon{cite some paper analyzing the approximating ReLU using polynomials.}
%\end{example}

%We verify our theoretical findings empirically in Figure~\ref{fig:improper}.
%\simon{@ruosong: two plots: Gaussian input: d=5 and 50. g=linear, quadratic, fourth-order poly, cosine, relu}
%\wei{try cubic too, show it doesn't work?}

%We want to emphasize that not able to learn ReLU function does not mean the over-parameterized two-layer neural network is a not a good model for learning as it already exhibits good performance on real world vision datasets.
%Therefore, we expect that assuming the underlying model is ReLU activated neural network may not be appropriate to model real world vision problems.

\begin{comment}
\begin{figure}[t]
	\label{fig:improper}
	\subfigure[$d=10$]
	{
		\label{fig:improper_d_5}
		\includegraphics[width=0.23\textwidth]{figs/mnist_spectrum.png}}   
	\subfigure[$d=5$]
	{
		\label{fig:improper_d_50}
		\includegraphics[width=0.23\textwidth]{figs/mnist_spectrum.png}}   
	\caption{[Place holder] Learning specific functions by over-parameterized two-layer neural networks. }
\end{figure}

\end{comment}

%\subsection{Proof Sketch of Theorem~\ref{thm:improper_learning}}
%\label{sec:proof_sketch_improper}

%\section{Proof Sketch}
%\label{sec:proof_sketch}
%\input{proof_sketch.tex}

\section{Conclusion}
\label{sec:con}

This paper shows how to give a fine-grained analysis of the optimization trajectory and the generalization ability of overparameterized two-layer neural networks trained by gradient descent. 
%The deep nets are randomly initialized and thus are similar to \textquotedblleft random kitchen sinks\textquotedblright ~\citep{rahimi2009weighted}. A small amount of training suffices to train the kitchen sink to fit interesting classes of 
%functions. 
We believe that our approach can also be useful in analyzing overparameterized deep neural networks and other machine learning models.

\section*{Acknowledgements}
 SA, WH and ZL acknowledge support from NSF, ONR, Simons Foundation, Schmidt Foundation, Mozilla Research, Amazon Research, DARPA and SRC.
 SSD acknowledges support from AFRL grant FA8750-17-2-0212 and DARPA D17AP00001.
 RW acknowledges support from ONR grant N00014-18-1-2562.
 Part of the work was done while SSD and RW were visiting the Simons Institute.

\bibliography{simonduref}
\bibliographystyle{icml2019}

\onecolumn
\newpage
\appendix

\section*{\Large Appendix}
\section{Experiment Setup}\label{sec:expdetails}
The architecture of our neural networks is as described in Section \ref{sec:setup}.
During the training process,  we fix the second layer and only optimize the first layer, following the setting in Section \ref{sec:setup}.
We fix the number of neurons to be $m = 10,000$ in all experiments. 
We train the neural network using (full-batch) gradient descent (GD), with a fixed learning rate $\eta = 10^{-3}$.
Our theory requires a small scaling factor $\kappa$ during the initialization (cf.~\eqref{eqn:random-init}).
We fix $\kappa = 10^{-2}$ in all experiments. 
We train the neural networks until the training $\ell_2$ loss converges.

We use two image datasets, the CIFAR dataset \cite{krizhevsky2009learning} and the MNIST dataset \cite{lecun1998gradient}, in our experiments. 
We only use the first two classes of images in the the CIFAR dataset and the MNIST dataset,
with $10,000$ training images and $2,000$ validation images in total for each dataset.
In both datasets, for each image $\vect x_i$, we set the corresponding label $y_i$ to be $+1$ if the image belongs to the first class, and $-1$ otherwise.
For each image $\vect x_i$ in the dataset, we normalize the image so that $\|\vect x_i\|_2 = 1$, following the setup in Section \ref{sec:setup}.

In the experiments reported in Figure~\ref{fig:generalization}, we choose a specific portion of (both training and test) data uniformly at random, and change their labels $y_i$ to $\unif\left(\left\{-1,1\right\}\right)$.

Our neural networks are trained using the PyTorch package \cite{paszke2017automatic}, using (possibly multiple) NVIDIA Tesla V100 GPUs.
\section{Background on Generalization and Rademacher Complexity} \label{app:rademacher}

%generalization
%\subsection{Generalization via Rademacher Complexity}

Consider a loss function $\ell: \R \times \R \to \R$. For a function $f: \R^d \to \R$, the \emph{population loss} over data distribution $\calD$ as well as the \emph{empirical loss} over $n$ samples $S = \{(\vect{x}_i, y_i)\}_{i=1}^n$ from $\calD$ are defined as:
\begin{equation*}
\begin{aligned}
L_\calD(f) &= \E_{(\vect x, y)\sim\calD}\left[ \ell(f(\vect x), y) \right],\\
L_S(f) &= \frac{1}{n}\sum_{i=1}^{n}\ell(f(\vect{x}_i),y_i).
\end{aligned}
\end{equation*}
Generalization error refers to the gap $L_\calD(f) - L_S(f)$ for the learned function $f$ given sample $S$.

Recall the standard definition of Rademacher complexity:
\begin{defn}\label{defn:rademacher}
%Given a sample $S = \left(\vect{x}_1,\ldots,\vect{x}_n\right)$, 
Given $n$ samples $S$, the empirical Rademacher complexity of a function class $\cF$ (mapping from $\R^d$ to $\R$) is defined as:
\[
\calR_S(\cF) = \frac{1}{n} \expect_{\bm \eps}\left[\sup_{f \in \cF}\sum_{i=1}^{n} \eps_i f(\vect{x}_i)\right],
\]
where $\bm\eps = \left(\eps_1,\ldots,\eps_n\right)^\top $ contains i.i.d. random variables drawn from the Rademacher distribution $\unif(\{1, -1\})$.
%, i.e., $\prob(\epsilon_i=1) = \prob(\eps_i=-1)=1/2$ for $i=1,\ldots,n$.
\end{defn}

Rademacher complexity directly gives an upper bound on generalization error (see e.g.~\citep{mohri2012foundations}):%\citep{koltchinskii2002empirical}.
%\simon{@zhiyuan, is this correct?}
\begin{thm}\label{thm:rad_generalization}
%Assume each data point is sampled i.i.d. from some distribution $\mu$, i.e.,
%\[
%(\vect{x}_i, y_i) \sim \mu \text{ for } i=1,\ldots,n.
%\]
%We denote $S = \left\{\vect{x}_i,y_i\right\}_{i=1}^n$.
%Given a function class $\cF$, for $f \in \cF$, we let $L_{tr}(f) = \frac{1}{n}\sum_{i=1}^{n}\ell(f(\vect{x}_i),y_i)$ and $L_{te}(f) = \expect_{(\vect{x},y) \sim \mu }\left[\ell(f(\vect{x}),y)\right]$. 
Suppose the loss function $\ell(\cdot,\cdot)$ is bounded in $[0, c]$ and is $\rho$-Lipschitz in the first argument. Then with probability at least $1-\delta$ over sample $S$ of size $n$:
\[
\sup_{f \in \cF} \left\{ L_{\calD}(f)-L_{S}(f) \right\} \le 2\rho \calR_S(\cF) + 3c \sqrt{\frac{\log(2/\delta)}{2n}}.
\]
%where $C>0$ is an absolute constant and $\widehat{R}_S(\cF)$ is the empirical Rademacher complexity of $\cF$.
\end{thm}
Therefore, as long as we can bound the Rademacher complexity of a certain function class that contains our learned predictor, we can obtain a generalization bound.
%In Section~\ref{sec:generalization} we use this strategy to derive a generalization bound for two-layer over-parameterized neural networks. 

\section{Proofs for Section~\ref{sec:rate}}
\label{app:proof_rate}

In this section we prove Theorem~\ref{thm:convergence_rate}.

We first show some technical lemmas. Most of them are already proved in~\citep{du2018provably} and we give proofs for them for completeness.

First, we have the following lemma which gives an upper bound on how much each weight vector can move during optimization.

\begin{lem} \label{lem:weight-vector-movement}
	Under the same setting as Theorem~\ref{thm:ssdu-converge}, i.e., $\lambda_0 = \lambda_{\min}(\mat H^\infty) >0$, $m = \Omega\left( \frac{n^6}{\lambda_0^4 \kappa^2 \delta^3 } \right)$ and $\eta = O\left( \frac{\lambda_0}{n^2} \right)$,
	with probability at least $1-\delta$ over the random initialization we have
	\begin{equation*}
	\norm{\vect w_r(k) - \vect w_r(0)}_2 \le \frac{4\sqrt n \norm{\vect y- \vect u(0)}_2}{\sqrt m \lambda_0}, \quad \forall r\in[m], \forall k \ge 0.
	\end{equation*}
\end{lem}
\begin{proof}
	From Theorem~\ref{thm:ssdu-converge} we know $ \norm{\vect y- \vect u(k+1)}_2 \le \sqrt{1 - \frac{\eta\lambda_0}{2}} \norm{\vect y- \vect u(k)}_2 \le \left( 1 - \frac{\eta\lambda_0}{4} \right) \norm{\vect y- \vect u(k)}_2$ for all $k\ge0$, which implies
	\begin{equation} \label{eqn:y-u(k)_bound}
	\norm{\vect y- \vect u(k)}_2 \le \left( 1 - \frac{\eta\lambda_0}{4} \right)^k \norm{\vect y- \vect u(0)}_2, \quad \forall k\ge 0.
	\end{equation}
	Recall the GD update rule:
	\begin{equation*}
	\vect w_r(k+1) - \vect w_r(k) = - \eta \frac{\partial \Phi(\mat W)}{\partial \vect{w}_r} \bigg|_{\mat W = \mat W(k)} = -\frac{\eta}{\sqrt m} a_r \sum_{i=1}^n (u_i(k) - y_i) \indict_{r, i}(k) \vect{x}_i,
	\end{equation*}
	which implies
	\begin{equation*}
	\norm{\vect w_r(k+1) - \vect w_r(k)}_2 \le \frac{\eta}{\sqrt m} \sum_{i=1}^n |u_i(k) - y_i| \le \frac{\eta \sqrt n}{\sqrt m} \norm{\vect y - \vect u(k)}_2.
	\end{equation*}
	Therefore, we have
	\begin{align*}
	\norm{ \vect w_r(k) - \vect w_r(0) }_2
	& \le \sum_{t=0}^{k-1} \norm{\vect w_r(t+1) - \vect w_r(t)}_2 \\
	& \le \sum_{t=0}^{k-1} \frac{\eta \sqrt n}{\sqrt m} \norm{\vect y - \vect u(k)}_2 \\
	& \le \frac{\eta \sqrt n}{\sqrt m}  \sum_{t=0}^{\infty} \left( 1 - \frac{\eta\lambda_0}{4} \right)^k \norm{\vect y- \vect u(0)}_2 \\
	&= \frac{4\sqrt n \norm{\vect y- \vect u(0)}_2}{\sqrt m \lambda_0},
	\end{align*}
	completing the proof.
\end{proof}

As a consequence of Lemma~\ref{lem:weight-vector-movement}, we can upper bound the norms of $\mat H(k) - \mat H(0)$ and $\mat Z(k) - \mat Z(0)$ for all $k$.
These bounds show that the matrices $\mat H$ and $\mat Z$ do not change much during optimization if $m$ is sufficiently large.

\begin{lem} \label{lem:H-and-Z-perturbation}
	Under the same setting as Theorem~\ref{thm:ssdu-converge}, with probability at least $1-4\delta$ over the random initialization, for all $k\ge 0$ we have:
	\begin{equation*}
	\begin{aligned}
	\norm{\mat H(k) - \mat H(0)}_F = O\left( \frac{n^3}{\sqrt m \lambda_0 \kappa \delta^{3/2}} \right), \\
	%= O\left( \frac{n^{5/2} \norm{\vect y- \vect u(0)}_2}{\sqrt m \lambda_0 \kappa \delta} \right), \\
	 \norm{\mat Z(k) - \mat Z(0)}_F = O\left( \frac{n}{\sqrt{m^{1/2} \lambda_0 \kappa \delta^{3/2}}} \right).
	 %= O\left( \sqrt{\frac{n^{3/2}  \norm{\vect y- \vect u(0)}_2}{\sqrt m \lambda_0 \kappa \delta}} \right).
	\end{aligned}
	\end{equation*}
\end{lem}
\begin{proof}
	Let $R = \frac{C n }{\sqrt m \lambda_0 \sqrt{\delta}}$ for some universal constant $C>0$.
	From Lemma~\ref{lem:weight-vector-movement} we know that with probability at least $1-\delta$ we have $\norm{\vect w_r(k) - \vect w_r(0)} \le R$ for all $r\in[m]$ and all $k \ge 0$.
	% Now we condition on this indeed happening.
	
	We define events
	\begin{equation} \label{eqn:A_{ri}-defn}
	A_{r, i} = \left\{ \left| \vect w_r(0)^\top \vect x_i \right| \le R \right\}, \quad i\in[n], r\in [m].
	\end{equation}
	%If $\indict\{A_{r, i}\} = 1$, then it is possible that $\vect w_r^\top \vect x_i$ changes sign during optimization. On the other hand, if $\indict\{A_{r, i}\} = 0$, $\vect w_r^\top \vect x_i$ will not change sign during optimization.
	Then it is easy to see that
	\begin{equation*}
	\indict\left\{ \indict_{r, i}(k) \not= \indict_{r, i}(0) \right\} \le  \indict\{A_{r, i}\} + \indict\{ \norm{\vect w_r(k) - \vect w_r(0)} > R \}  , \quad \forall i\in[n], \forall r\in[m], \forall k\ge0.
	\end{equation*}
	
	Then for any $i, j\in[n]$ we have
	\begin{equation} \label{eqn:H-pertubation-proof}
	\begin{aligned}
	\left| \mat H_{ij}(k) - \mat H_{ij}(0) \right|
	&= \left| \frac{\vect{x}_i^\top \vect{x}_j}{m} \sum_{r=1}^{m} \left( \indict_{r,i}(k) \indict_{r,j}(k) - \indict_{r,i}(0) \indict_{r,j}(0) \right) \right| \\
	&\le \frac1m \sum_{r=1}^{m} \left( \indict\left\{ \indict_{r, i}(k) \not= \indict_{r, i}(0) \right\} + \indict\left\{ \indict_{r, j}(k) \not= \indict_{r, j}(0)  \right\} \right) \\
	&\le \frac1m \sum_{r=1}^{m} \left( \indict\{A_{r, i}\}+\indict\{A_{r, j}\} + 2 \indict\{ \norm{\vect w_r(k) - \vect w_r(0)} > R \} \right).
	\end{aligned}
	\end{equation}
	Next, notice that $\vect w_r(0)^\top \vect x_i$ has the same distribution as $\calN(0, \kappa^2)$. So we have
	\begin{equation} \label{eqn:A_{ri}-bound}
	\E[\indict\{A_{r, i}\}] = \Pr_{z\sim \calN(0, \kappa^2)} \left[ |z| \le R \right]
	= \int_{-R}^{R} \frac{1}{\sqrt{2\pi}\kappa} e^{-x^2/2\kappa^2} dx
	\le \frac{2R}{\sqrt{2\pi}\kappa}.
	\end{equation}
	Now taking the expectation of~\eqref{eqn:H-pertubation-proof}, we have
	\begin{align*}
	\E\left[ \left| \mat H_{ij}(k) - \mat H_{ij}(0) \right| \right]
	&\le \frac1m \sum_{r=1}^{m} \left( \E[\indict\{A_{r, i}\}] + \E[\indict\{A_{r, j}\}] + 2 \indict\{ \norm{\vect w_r(k) - \vect w_r(0)} > R \} \right)\\
	&\le \frac{4R}{\sqrt{2\pi}\kappa} + \frac2m \sum_{r=1}^{m}  \E\left[ \indict\{ \norm{\vect w_r(k) - \vect w_r(0)} > R \} \right]\\
	&\le  \frac{4R}{\sqrt{2\pi}\kappa} + \frac2m \delta, \qquad \forall i, j \in[n],
	\end{align*}
	which implies
	\begin{equation*}
	\E \left[ \norm{\mat H(k) - \mat H(0)}_F \right]
	\le \E\left[ \sum_{i, j=1}^n \left| \mat H_{ij}(k) - \mat H_{ij}(0) \right| \right]
	\le \frac{4n^2R}{\sqrt{2\pi}\kappa} + \frac{2n^2\delta}{m}.
	\end{equation*}
	Then from Markov's inequality we know that with probability at least $1-\delta$ we have $\norm{\mat H(k) - \mat H(0)}_F \le \frac{4n^2R}{\sqrt{2\pi}\kappa\delta} + \frac{2n^2}{m} = O\left( \frac{n^3}{\sqrt m \lambda_0 \kappa \delta^{3/2}} \right) $. % Then the first part of the lemma follows from $\norm{\vect y - \vect u(0)}_2 = O(\sqrt{n/\delta})$ (Theorem~\ref{thm:ssdu-converge}).
	 %which proves the first part of the lemma.
	
	To bound $\norm{\mat Z(k) - \mat Z(0)}_F$, we have
	\begin{align*}
	\E\left[ \norm{\mat Z(k) - \mat Z(0)}_F^2 \right]
	&=  \E\left[ \frac1m \sum_{i=1}^n \sum_{r=1}^m \left( \indict_{r, i}(k) - \indict_{r, i}(0) \right)^2  \right]\\
	&= \frac1m  \sum_{i=1}^n \sum_{r=1}^m \E\left[  \indict\{ \indict_{r, i}(k) \not= \indict_{r, i}(0) \}  \right]\\
	&\le \frac1m  \sum_{i=1}^n \sum_{r=1}^m \E\left[  \indict\{ A_{r, i} \} + \indict\{ \norm{\vect w_r(k) - \vect w_r(0)} > R \}  \right]\\
	&\le \frac1m \cdot mn \cdot \frac{2R}{\sqrt{2\pi}\kappa} + \frac nm \delta\\
	&= \frac{2nR}{\sqrt{2\pi}\kappa} + \frac{n\delta}{m}.
	\end{align*}
	Using Markov's inequality, with probability at least $1-\delta$ we have $\norm{\mat Z(k) - \mat Z(0)}_F^2 \le \frac{2nR}{\sqrt{2\pi}\kappa\delta} + \frac nm = O\left( \frac{n^{2}  }{\sqrt m \lambda_0\kappa\delta^{3/2}} \right)$,
	proving the second part of the lemma.
\end{proof}

The next lemma shows that $\mat H(0)$ is close to $\mat H^\infty$.

\begin{lem} \label{lem:H(0)-H^inf}
	With probability at least $1-\delta$, we have $\norm{\mat H(0) - \mat H^\infty}_F = O\left( \frac{n\sqrt{\log \frac{n}{\delta}}}{\sqrt m} \right)$.
\end{lem}
\begin{proof}
	For all $i, j \in[n]$, $\mat{H}_{ij}(0) = \frac{\vect{x}_i^\top \vect{x}_j}{m} \sum_{r=1}^{m}\indict_{r,i}(0)\indict_{r,j}(0)$ is the average of $m$ i.i.d. random variables bounded in $[0, 1]$ with expectation $\mat H^\infty_{ij}$.
	Thus by Hoeffding's inequality, with probability at least $1-\delta'$ we have
	$$\left| \mat{H}_{ij}(0) - \mat H^\infty_{ij} \right| \le \sqrt{\frac{\log ({2}/{\delta'})}{2m}}.$$
	Letting $\delta' = \delta/n^2$ and applying a union bound over all $i, j \in [n]$, we know that with probability at least $1-\delta$:
	\begin{align*}
	\norm{\mat H(0) - \mat H^\infty}_F^2 \le n^2 \cdot \frac{\log (2n^2/\delta)}{2m}.
	\end{align*}
	This completes the proof.
\end{proof}

Now we are ready to prove Theorem~\ref{thm:convergence_rate}.

\begin{proof}[Proof of Theorem~\ref{thm:convergence_rate}]
	We assume that all the high-probability (``with probability $1-\delta$'') events happen. By a union bound at the end, the success probability is at least $1-\Omega(\delta)$, and then we can rescale $\delta$ by a constant such that the success probability is at least $1-\delta$.
	
	The main idea is to show that the dynamics of $\left\{\vect u(k) \right\}_{k=0}^\infty$ is close to that of $\left\{\tilde{\vect u}(k) \right\}_{k=0}^\infty$.
	We have
	\begin{equation} \label{eqn:u-dynamics-1}
	u_i(k+1) - u_i(k) = \frac{1}{\sqrt m} \sum_{r=1}^m a_r \left[ \sigma\left( \vect w_r(k+1)^\top \vect x_i \right)  -  \sigma\left( \vect w_r(k)^\top \vect x_i \right) \right].
	\end{equation}
	For each $i\in[n]$, we partition all $m$ neurons into two parts:
	\[
	S_i = \left\{ r\in[m] : \indict\{A_{r, i}\} = 0  \right\}
	\]
	and
	\[
	\bar{S_i} = \left\{ r\in[m] : \indict\{A_{r, i}\} = 1 \right\},
	\]
	where $A_{r, i}$ is defined in~\eqref{eqn:A_{ri}-defn}.
	We know from the proof of Lemma~\ref{lem:H-and-Z-perturbation} that all neurons in $S_i$ will not change activation pattern on data-point $\vect x_i$ during optimization, i.e.,
	\[
	r\in S_i \Longrightarrow \indict_{r, i}(k) = \indict_{r, i}(0), \forall k\ge0.
	\]
	
	From~\eqref{eqn:A_{ri}-bound} we know
	\begin{align*}
	\E\left[ |\bar{S_i}| \right] = \E\left[ \sum_{r=1}^m \indict\{A_{r, i}\} \right]
	\le  \frac{8\sqrt{mn} \norm{\vect y- \vect u(0)}_2}{\sqrt{2\pi}\kappa \lambda_0}
	= O\left( \frac{\sqrt{m} n }{\kappa \lambda_0 \sqrt\delta} \right),
	\end{align*}
	where we have used $\norm{\vect y - \vect u(0)}_2 = O(\sqrt{n/\delta})$ (Theorem~\ref{thm:ssdu-converge}).
	Then we know $\E\left[ \sum_{i=1}^n |\bar{S_i}| \right]  = O\left( \frac{\sqrt{m} n^2 }{\kappa \lambda_0 \sqrt\delta} \right)$. Therefore with probability at least $1-\delta$ we have
	\begin{equation} \label{eqn:S_i_bar-bound}
	\sum_{i=1}^n |\bar{S_i}| = O\left( \frac{\sqrt{m} n^2 }{\kappa \lambda_0 \delta^{3/2}} \right).
	\end{equation}
	
	We write~\eqref{eqn:u-dynamics-1} as
	\begin{equation} \label{eqn:u-dynamics-2}
	\begin{aligned}
	&u_i(k+1) - u_i(k) \\
	=\ & \frac{1}{\sqrt m} \sum_{r\in S_i} a_r \left[ \sigma\left( \vect w_r(k+1)^\top \vect x_i \right)  -  \sigma\left( \vect w_r(k)^\top \vect x_i \right) \right]
	+ \frac{1}{\sqrt m} \sum_{r\in \bar{S_i}} a_r \left[ \sigma\left( \vect w_r(k+1)^\top \vect x_i \right)  -  \sigma\left( \vect w_r(k)^\top \vect x_i \right) \right].
	\end{aligned}
	\end{equation}
	We consider the two terms in~\eqref{eqn:u-dynamics-2} separately.
	We denote the second term as $\epsilon_i(k)$ and treat it as a perturbation term, which we bound as
	\begin{equation} \label{eqn:u-dynamics-second-term}
	\begin{aligned}
	\abs{\epsilon_i(k)} = \, &\left| \frac{1}{\sqrt m} \sum_{r\in \bar{S_i}} a_r \left[ \sigma\left( \vect w_r(k+1)^\top \vect x_i \right)  -  \sigma\left( \vect w_r(k)^\top \vect x_i \right) \right] \right| \\
	\le\,& \frac{1}{\sqrt m} \sum_{r\in \bar{S_i}} \left| \vect w_r(k+1)^\top \vect x_i  - \vect w_r(k)^\top \vect x_i  \right| \\
	\le\,& \frac{1}{\sqrt m} \sum_{r\in \bar{S_i}} \norm{\vect w_r(k+1) - \vect w_r(k) }_2  \\
	=\,& \frac{1}{\sqrt m} \sum_{r\in \bar{S_i}} \norm{\frac{\eta}{\sqrt m} a_r \sum_{j=1}^n (u_j(k) - y_j) \indict_{r, j}(k) \vect{x}_j }_2  \\
	\le\,& \frac{\eta}{m} \sum_{r\in \bar{S_i}} \sum_{j=1}^n |u_j(k) - y_j|   \\
	\le\,& \frac{\eta \sqrt n |\bar{S_i}|}{m} \norm{\vect u(k) - \vect y}_2.
	\end{aligned}
	\end{equation}
	For the first term in~\eqref{eqn:u-dynamics-2}, we have
	\begin{equation} \label{eqn:u-dynamics-first-term}
	\begin{aligned}
	&\frac{1}{\sqrt m} \sum_{r\in S_i} a_r \left[ \sigma\left( \vect w_r(k+1)^\top \vect x_i \right)  -  \sigma\left( \vect w_r(k)^\top \vect x_i \right) \right] \\
	=\,& \frac{1}{\sqrt m} \sum_{r\in S_i} a_r \indict_{r, i}(k) \left(  \vect w_r(k+1)  -  \vect w_r(k)\right)^\top \vect x_i   \\
	=\,& \frac{1}{\sqrt m} \sum_{r\in S_i} a_r \indict_{r, i}(k) \left( - \frac{\eta}{\sqrt m} a_r \sum_{j=1}^n (u_j(k) - y_j) \indict_{r, j}(k) \vect{x}_j  \right)^\top \vect x_i   \\
	=\,& -\frac{\eta}{m} \sum_{j=1}^n (u_j(k) - y_j) \vect{x}_j^\top\vect x_i \sum_{r\in S_i}  \indict_{r, i}(k) \indict_{r, j}(k) \\
	=\,& -\eta  \sum_{j=1}^n (u_j(k) - y_j) \mat H_{ij}(k) + \epsilon_i'(k),
	\end{aligned}
	\end{equation}
	where $\epsilon_i'(k) = \frac{\eta}{m} \sum_{j=1}^n (u_j(k) - y_j) \vect{x}_j^\top\vect x_i \sum_{r\in \bar{S_i}}  \indict_{r, i}(k) \indict_{r, j}(k)$ is regarded as perturbation:
	\begin{equation}  \label{eqn:u-dynamics-perturbation-1}
	\begin{aligned}
	\abs{\epsilon_i'(k)} \le \frac{\eta}{m} \abs{\bar{S_i}} \sum_{j=1}^n \abs{u_j(k) - y_j} 
	\le \frac{\eta \sqrt n \abs{\bar{S_i}} }{m} \norm{\vect u(k) - \vect y}_2  .
	\end{aligned}
	\end{equation}
	Combining \eqref{eqn:u-dynamics-2}, \eqref{eqn:u-dynamics-second-term}, \eqref{eqn:u-dynamics-first-term}, \eqref{eqn:u-dynamics-perturbation-1}, we have
	\begin{equation*}
	\begin{aligned}
	u_i(k+1) - u_i(k) =  -\eta  \sum_{j=1}^n (u_j(k) - y_j) \mat H_{ij}(k) + \epsilon_i'(k) + \epsilon_i(k),
	\end{aligned}
	\end{equation*}	
	which gives
	\begin{equation} \label{eqn:u-dynamics-3}
	\vect u(k+1) - \vect u(k) = - \eta \vect H(k) (\vect u(k) - \vect y) + \bm{\epsilon}(k),
	\end{equation}
	where $\bm{\epsilon}(k) \in \R^n$ can be bounded as
	\begin{equation*}
	\begin{aligned}
	\norm{\bm{\epsilon}(k)}_2 
	&\le \norm{\bm{\epsilon}(k)}_1 = \sum_{i=1}^n \abs{ \epsilon_i(k) + \epsilon_i'(k) }
	\le \sum_{i=1}^n \frac{2 \eta \sqrt n |\bar{S_i}|}{m} \norm{\vect u(k) - \vect y}_2
	= O\left( \frac{\sqrt{m} n^2 }{\kappa \lambda_0 \delta^{3/2}} \right)   \frac{2 \eta \sqrt n }{m} \norm{\vect u(k) - \vect y}_2 \\
	&= O\left( \frac{\eta  n^{5/2} }{\sqrt m \kappa \lambda_0 \delta^{3/2}} \right)  \norm{\vect u(k) - \vect y}_2.
	\end{aligned}
	\end{equation*}
	Here we have used~\eqref{eqn:S_i_bar-bound}.
	
	Next, since $\mat H(k)$ is close to $\mat H^\infty$ according to Lemmas~\ref{lem:H-and-Z-perturbation} and~\ref{lem:H(0)-H^inf}, we rewrite~\eqref{eqn:u-dynamics-3} as
	\begin{equation} \label{eqn:u-dynamics-4}
	\begin{aligned}
	\vect u(k+1) - \vect u(k) = - \eta \mat H^\infty (\vect u(k) - \vect y) + \bm{\zeta}(k),
	\end{aligned}
	\end{equation}
	where $\bm{\zeta}(k) = \eta (\mat H^\infty - \mat H(k)) (\vect u(k) - \vect y) + \bm{\epsilon}(k)$.
	From Lemmas~\ref{lem:H-and-Z-perturbation} and~\ref{lem:H(0)-H^inf} we have
	\begin{equation*}
	\begin{aligned}
	\norm{\mat H^\infty - \mat H(k)}_2
	&\le \norm{\mat H^\infty - \mat H(k)}_F
	 \le \norm{\mat H(0) - \mat H(k)}_F + \norm{\mat H(0) - \mat H^\infty}_2  \\
	 &= O\left( \frac{n^3}{\sqrt m \lambda_0 \kappa \delta^{3/2}} \right) + O\left( \frac{n\sqrt{\log \frac{n}{\delta}}}{\sqrt m} \right)
	 = O\left( \frac{n^3}{\sqrt m \lambda_0 \kappa \delta^{3/2}} \right).
	 \end{aligned}
	\end{equation*}
	Therefore we can bound $\bm{\zeta}(k)$ as
	\begin{equation} \label{eqn:u-dynamics-perturbation-2}
	\begin{aligned}
	\norm{\bm{\zeta}(k)}_2
	&\le \eta \norm{\mat H^\infty - \mat H(k)}_2 \norm{\vect u(k) - \vect y}_2 + \norm{\bm{\epsilon}(k)}_2 \\
	&= O\left(  \frac{\eta n^3}{\sqrt m \lambda_0 \kappa \delta^{3/2}} \right) \norm{\vect u(k) - \vect y}_2  + O\left( \frac{\eta  n^{5/2} }{\sqrt m \kappa \lambda_0 \delta^{3/2}} \right)  \norm{\vect u(k) - \vect y}_2 \\
	&= O\left(  \frac{\eta n^3}{\sqrt m \lambda_0 \kappa \delta^{3/2}} \right) \norm{\vect u(k) - \vect y}_2  .
	\end{aligned}
	\end{equation}
	
	Finally, applying~\eqref{eqn:u-dynamics-4} recursively, we get
	\begin{equation} \label{eqn:u-expression}
	\begin{aligned}
	\vect u(k) - \vect y &= (\mat I - \eta \mat H^\infty)^k (\vect u(0) - \vect{y}) + \sum_{t=0}^{k-1} (\mat I - \eta \mat H^\infty)^t \bm{\zeta}(k-1-t) \\
	&= -(\mat I - \eta \mat H^\infty)^k \vect{y} + (\mat I - \eta \mat H^\infty)^k \vect u(0)  + \sum_{t=0}^{k-1} (\mat I - \eta \mat H^\infty)^t \bm{\zeta}(k-1-t).
	\end{aligned}
	\end{equation}
	Note that $\mat I - \eta \mat H^\infty$ is positive semidefinite, because we have $\norm{\mat H^\infty}_2 \le \tr\left[\mat H^\infty\right] = \frac n2$ and $\eta = O\left( \frac{\lambda_0}{n^2} \right) = O\left( \frac{\lambda_{\min}(\mat H^\infty)}{\norm{\mat H^\infty}_2^2} \right) \le \frac{1}{\norm{\mat H^\infty}_2}$.
	This implies $\norm{\mat I - \eta \mat H^\infty}_2 \le 1 - \eta \lambda_0$.
		
	Now we study the three terms in~\eqref{eqn:u-expression} separately.
	The first term is exactly $\tilde{\vect u}(k) - \vect y$ (see Section~\ref{sec:proof_sketch_rate}), and in Section~\ref{sec:proof_sketch_rate} we have shown that
	\begin{equation} \label{eqn:u-expression-first-term}
	\norm{-(\mat I - \eta \mat H^\infty)^k \vect{y}}_2 = \sqrt{\sum_{i=1}^n (1-\eta\lambda_i)^{2k} (\vect v_i^\top \vect y)^2}.
	\end{equation}
	
	The second term in~\eqref{eqn:u-expression} is small as long as $\vect u(0)$ is small, which is the case when the magnitude of initialization $\kappa$ is set to be small.
	Formally, each $u_i(0)$ has zero mean and variance $O(\kappa^2)$, which means $\E\left[(u_i(0))^2\right] = O(\kappa^2)$. This implies $\E\left[\norm{\vect u(0)}^2\right] = O(n\kappa^2)$, and by Markov's inequality we have $\norm{\vect u(0)}^2 \le \frac{n\kappa^2}{\delta}$ with probability at least $1-\delta$.
	Therefore we have
	\begin{equation} \label{eqn:u-expression-second-term}
	\norm{(\mat I - \eta \mat H^\infty)^k \vect u(0) }_2
	\le \norm{(\mat I - \eta \mat H^\infty)^k }_2 \norm{\vect u(0)}_2
	\le (1-\eta \lambda_0)^k O\left( \sqrt n \kappa / \delta \right).
	\end{equation} 
	
	The third term in~\eqref{eqn:u-expression} can be bounded using~\eqref{eqn:u-dynamics-perturbation-2}.
	Also note that we have $\norm{\vect u(k) - \vect y}_2 \le \left( 1 - \frac{\eta\lambda_0}{4} \right)^k \norm{\vect u(0) - \vect y}_2 = \left( 1 - \frac{\eta\lambda_0}{4} \right)^k O(\sqrt{n/\delta})$ (Theorem~\ref{thm:ssdu-converge}).
	Therefore we have
	\begin{equation} \label{eqn:u-expression-third-term}
	\begin{aligned}
	\norm{\sum_{t=0}^{k-1} (\mat I - \eta \mat H^\infty)^t \bm{\zeta}(k-1-t)}_2
	&\le \sum_{t=0}^{k-1} \norm{ \mat I - \eta \mat H^\infty}_2^t \norm{\bm{\zeta}(k-1-t)}_2 \\
	&\le \sum_{t=0}^{k-1} (1-\eta \lambda_0)^t O\left(  \frac{\eta n^3}{\sqrt m \lambda_0 \kappa \delta^{3/2}} \right) \norm{\vect u(k-1-t) - \vect y}_2  \\
	&\le \sum_{t=0}^{k-1} (1-\eta \lambda_0)^t O\left(  \frac{\eta n^3}{\sqrt m \lambda_0 \kappa \delta^{3/2}} \right) \left( 1 - \frac{\eta\lambda_0}{4} \right)^{k-1-t} O(\sqrt{n/\delta})  \\
	&\le k \left( 1 - \frac{\eta\lambda_0}{4} \right)^{k-1}   O\left(  \frac{\eta n^{7/2}}{\sqrt m \lambda_0 \kappa \delta^{2}} \right) .
	\end{aligned}
	\end{equation}
	
	Combining \eqref{eqn:u-expression}, \eqref{eqn:u-expression-first-term}, \eqref{eqn:u-expression-second-term} and \eqref{eqn:u-expression-third-term}, we obtain
	\begin{align*}
	\norm{\vect u(k) - \vect y}_2
	&= \sqrt{\sum_{i=1}^n (1-\eta\lambda_i)^{2k} (\vect v_i^\top \vect y)^2} \pm O\left( (1-\eta \lambda_0)^k \frac{ \sqrt n \kappa }{ \delta}  + k \left( 1 - \frac{\eta\lambda_0}{4} \right)^{k-1}    \frac{\eta n^{7/2}}{\sqrt m \lambda_0 \kappa \delta^{2}}\right) \\
	&= \sqrt{\sum_{i=1}^n (1-\eta\lambda_i)^{2k} (\vect v_i^\top \vect y)^2} \pm O\left(  \frac{ \sqrt n \kappa }{ \delta} + \frac{1}{\eta\lambda_0} \cdot   \frac{\eta n^{7/2}}{\sqrt m \lambda_0 \kappa \delta^{2}}\right)\\
		&= \sqrt{\sum_{i=1}^n (1-\eta\lambda_i)^{2k} (\vect v_i^\top \vect y)^2} \pm O\left(  \frac{ \sqrt n \kappa }{ \delta} +  \frac{ n^{7/2}}{\sqrt m \lambda_0^2 \kappa \delta^{2}}\right),
	\end{align*}
	where we have used $\max\limits_{k\ge 0} \left\{ k (1-\eta\lambda_0/4)^{k-1} \right\} = O(1/(\eta\lambda_0))$.
	From our choices of $\kappa$ and $m$, the above error term is at most $\epsilon$.
	 This completes the proof of Theorem~\ref{thm:convergence_rate}.
\end{proof}

\section{Proofs for Section~\ref{sec:generalization}}
\label{app:proof_generalization}

\subsection{Proof of Lemma~\ref{lem:distance_bounds}} \label{app:proof-lem:distance_bounds}

\begin{proof}[Proof of Lemma~\ref{lem:distance_bounds}]
	We assume that all the high-probability (``with probability $1-\delta$'') events happen. By a union bound at the end, the success probability is at least $1-\Omega(\delta)$, and then we can rescale $\delta$ by a constant such that the success probability is at least $1-\delta$.
	
	The first part of Lemma~\ref{lem:distance_bounds} is proved as Lemma~\ref{lem:weight-vector-movement} (note that $\norm{\vect y - \vect u(0)}_2 = O\left( \sqrt{n/\delta} \right)$). Now we prove the second part.
	
	Recall the update rule~\eqref{eqn:gd} for $\mat W$:
	\begin{equation} \label{eqn:gd-repeated}
	\vectorize{\mat W(k+1)} = \vectorize{\mat W(k)} - \eta \mat Z(k) (\vect u(k) - \vect y).
	\end{equation}
	According to the proof of Theorem~\ref{thm:convergence_rate} (\eqref{eqn:u-expression}, \eqref{eqn:u-expression-second-term} and \eqref{eqn:u-expression-third-term}) we can write
	\begin{equation} \label{eqn:u-expression-2}
	\vect u(k) - \vect y = -(\mat I - \eta \mat H^\infty)^k \vect y + \vect e(k),
	\end{equation}
	where
	\begin{equation} \label{eqn:e(k)-bound}
	\norm{\vect e(k)} = O\left(  (1-\eta \lambda_0)^k \cdot \frac{\sqrt n \kappa}{ \delta}  + k\left(1-\frac{\eta \lambda_0}{4}\right)^{k-1}\cdot \frac{\eta n^{7/2}}{\sqrt m \lambda_0 \kappa \delta^{2}}  \right).
	\end{equation}
	Plugging \eqref{eqn:u-expression-2} into \eqref{eqn:gd-repeated} and taking a sum over $k = 0, 1, \ldots, K-1$, we get:
	\begin{equation} \label{eqn:W-movement}
	\begin{aligned}
	&\vectorize{\mat W(K)} - \vectorize{\mat W(0)} \\
	=\,& \sum_{k=0}^{K-1} \left( \vectorize{\mat W(k+1)} - \vectorize{\mat W(k)}  \right) \\
	=\,& - \sum_{k=0}^{K-1} \eta \mat Z(k) (\vect u(k) - \vect y) \\
	=\,&  \sum_{k=0}^{K-1} \eta \mat Z(k) \left((\mat I - \eta \mat H^\infty)^k \vect y - \vect e(k)\right) \\
	=\,&  \sum_{k=0}^{K-1} \eta \mat Z(k) (\mat I - \eta \mat H^\infty)^k \vect y  - \sum_{k=0}^{K-1} \eta \mat Z(k) \vect e(k) \\
	=\,&  \sum_{k=0}^{K-1} \eta \mat Z(0) (\mat I - \eta \mat H^\infty)^k \vect y + \sum_{k=0}^{K-1} \eta (\mat Z(k) - \mat Z(0)) (\mat I - \eta \mat H^\infty)^k \vect y  - \sum_{k=0}^{K-1} \eta \mat Z(k)  \vect e(k).
	\end{aligned}
	\end{equation}
	
	The second and the third terms in~\eqref{eqn:W-movement} are considered perturbations, and we can upper bound their norms easily.
	For the second term, using $\norm{\mat Z(k) - \mat Z(0)}_F = O\left( \frac{n}{\sqrt{m^{1/2} \lambda_0 \kappa \delta^{3/2}}} \right)$ (Lemma~\ref{lem:H-and-Z-perturbation}), we have:
	\begin{equation} \label{eqn:W-movement-perturbation-1}
	\begin{aligned}
	& \norm{\sum_{k=0}^{K-1} \eta (\mat Z(k) - \mat Z(0)) (\mat I - \eta \mat H^\infty)^k \vect y}_2 \\
	\le\,& \sum_{k=0}^{K-1} \eta \cdot O\left( \frac{n}{\sqrt{m^{1/2} \lambda_0 \kappa \delta^{3/2}}} \right) \norm{\mat I - \eta \mat H^\infty}_2^k \norm{\vect y}_2 \\
	\le\,&  O\left( \frac{\eta n}{\sqrt{m^{1/2} \lambda_0 \kappa \delta^{3/2}}} \right)  \sum_{k=0}^{K-1} (1-\eta\lambda_0)^k \sqrt n \\
	=\,& O\left( \frac{ n^{3/2}}{ \sqrt{m^{1/2} \lambda_0^3 \kappa \delta^{3/2}}} \right) . 
	\end{aligned}
	\end{equation}
	
	For the third term in \eqref{eqn:W-movement}, we use $\norm{\mat Z(k)}_F \le \sqrt n$ and \eqref{eqn:e(k)-bound} to get:
	\begin{equation} \label{eqn:W-movement-perturbation-2}
	\begin{aligned}
	&\norm{\sum_{k=0}^{K-1} \eta \mat Z(k)  \vect e(k)}_2\\
	\le\,& \sum_{k=0}^{K-1} \eta \sqrt n \cdot O\left(  (1-\eta \lambda_0)^k \cdot \frac{\sqrt n \kappa}{ \delta}  + k\left(1-\frac{\eta \lambda_0}{4}\right)^{k-1}\cdot \frac{\eta n^{7/2}}{\sqrt m \lambda_0 \kappa \delta^{2}}  \right) \\
	=\,& O\left( \frac{\eta  n \kappa}{ \delta} \sum_{k=0}^{K-1}     (1-\eta \lambda_0)^k 
		+ \frac{\eta^2  n^4}{\sqrt m \lambda_0 \kappa \delta^{2}} \sum_{k=0}^{K-1} k\left(1-\frac{\eta \lambda_0}{4}\right)^{k-1} \right) \\
	=\,& O\left( \frac{  n \kappa}{\lambda_0 \delta} + \frac{  n^4}{\sqrt m \lambda_0^3 \kappa \delta^{2}}  \right).
	\end{aligned}
	\end{equation}
	
	Define $\mat T = \eta \sum_{k=0}^{K-1} (\mat I - \eta \mat H^\infty)^k$.
	For the first term in \eqref{eqn:W-movement}, using $\norm{\mat H(0) - \mat H^\infty}_F = O\left( \frac{n\sqrt{\log \frac{n}{\delta}}}{\sqrt m} \right)$ (Lemma~\ref{lem:H(0)-H^inf}) we have
	\begin{equation} \label{eqn:W-movement-main-term-bound-interim}
	\begin{aligned}
	&\norm{\sum_{k=0}^{K-1} \eta \mat Z(0) (\mat I - \eta \mat H^\infty)^k \vect y}_2^2 \\
	=\,& \norm{ \mat Z(0) \mat T  \vect y}_2^2 \\
	=\,&  \vect y^\top \mat T \mat Z(0)^\top \mat Z(0) \mat T \vect y \\
	=\,&  \vect y^\top \mat T \mat H(0)  \mat T \vect y \\
	\le\,&  \vect y^\top \mat T \mat H^\infty  \mat T \vect y  +    \norm{\mat H(0) - \mat H^\infty}_2  \norm{\mat T}_2^2 \norm{\vect y}_2^2 \\
	\le \,&  \vect y^\top \mat T \mat H^\infty  \mat T \vect y  +    O\left( \frac{n\sqrt{\log \frac{n}{\delta}}}{\sqrt m} \right) \cdot  \left( \eta \sum_{k=0}^{K-1} (\mat I - \eta \lambda_0)^k\right)^2 n \\
	= \,&  \vect y^\top \mat T \mat H^\infty  \mat T \vect y  +    O\left( \frac{n^2\sqrt{\log \frac{n}{\delta}}}{\sqrt m \lambda_0^2} \right) .
	\end{aligned}
	\end{equation}
	Let the eigen-decomposition of $\mat H^\infty$ be $\mat H^\infty = \sum_{i=1}^n \lambda_i \vect v_i \vect v_i^\top$.
	Since $\mat T$ is a polynomial of $\mat H^\infty$, it has the same set of eigenvectors as $\mat H^\infty$, and we have
	\begin{align*}
	\mat T  = \sum_{i=1}^n \eta \sum_{k=0}^{K-1} (1 - \eta \lambda_i)^k \vect v_i \vect v_i^\top
	= \sum_{i=1}^n \frac{1 - (1 - \eta \lambda_i)^{K}}{\lambda_i} \vect v_i \vect v_i^\top.
	\end{align*}
	It follows that
	\begin{align*}
	\mat T \mat H^\infty  \mat T
	= \sum_{i=1}^n \left( \frac{1 - (1 - \eta \lambda_i)^{K}}{\lambda_i} \right)^2 \lambda_i  \vect v_i \vect v_i^\top
	\preceq \sum_{i=1}^n \frac{1}{\lambda_i}\vect v_i \vect v_i^\top
	= \left( \mat H^\infty \right)^{-1}.
	\end{align*}
	Plugging this into \eqref{eqn:W-movement-main-term-bound-interim}, we get
	\begin{equation} \label{eqn:W-movement-main-term-bound}
	\norm{\sum_{k=0}^{K-1} \eta \mat Z(0) (\mat I - \eta \mat H^\infty)^k \vect y}_2
	\le \sqrt{\vect y^\top (\mat H^\infty)^{-1} \vect y  +    O\left( \frac{n^2\sqrt{\log \frac{n}{\delta}}}{\sqrt m \lambda_0^2} \right) }
	\le \sqrt{\vect y^\top (\mat H^\infty)^{-1} \vect y}  +    O\left( \sqrt{\frac{n^2\sqrt{\log \frac{n}{\delta}}}{\sqrt m \lambda_0^2} } \right) .
	\end{equation}
	
	Finally, plugging the three bounds \eqref{eqn:W-movement-perturbation-1}, \eqref{eqn:W-movement-perturbation-2} and \eqref{eqn:W-movement-main-term-bound} into \eqref{eqn:W-movement}, we have
	\begin{align*}
	& \norm{\mat W(K) - \mat W(0)}_F \\
	=\,& \norm{\vectorize{\mat W(K)} - \vectorize{\mat W(0)} }_2 \\
	\le\,& \sqrt{\vect y^\top (\mat H^\infty)^{-1} \vect y}  +    O\left( \sqrt{\frac{n^2\sqrt{\log \frac{n}{\delta}}}{\sqrt m \lambda_0^2} } \right)
		+ O\left( \frac{ n^{3/2}}{ \sqrt{m^{1/2} \lambda_0^3 \kappa \delta^{3/2}}} \right)
		+ O\left( \frac{  n \kappa}{\lambda_0 \delta} + \frac{  n^4}{\sqrt m \lambda_0^3 \kappa \delta^{2}}  \right) \\
	=\,& \sqrt{\vect y^\top (\mat H^\infty)^{-1} \vect y}  + O\left( \frac{  n \kappa}{\lambda_0 \delta} \right) + \frac{\poly\left(n, \frac{1}{\lambda_0}, \frac{1}{\delta} \right)}{m^{1/4} \kappa^{1/2}} .
	\end{align*}
	This finishes the proof of Lemma~\ref{lem:distance_bounds}.
\end{proof}

\subsection{Proof of Lemma~\ref{lem:rad_dist_func_class}} \label{app:proof-lem:rad_dist_func_class}

%
%Define function class 
%
%\[\cF^{W(0),\va}_{R,B} = \{f:\reals^d \to \reals \mid f(\vx) = \sum_{r=1}^m\frac{ a_r}{\sqrt{m}} \relu{\vw_r^\top \vx}, \norm{\vw_r(0)-\vw}_{2}\le R,\norm{\vW(0)-\vW}_F\le B \}.\]
%
%
%\begin{thm}
%For any fixed $\{\vx_i\}_{i=1}^n$ s.t. $\norm{\vx_i}\le 1$,$\forall i\in [n]$, suppose $\{\vw_r(0)\}_{r=1}^m$,$\{a_r\}_{r=1}^m$ are independently sampled from $N(\bm{0}, I_d)$, $\textrm{Unif}[-1,1]$ respectively. Then w.h.p.  the empirical rademacher complexity has the following upper bound 
% \[\rc{\cF^{\vW(0),\va}_{R,B}; \vX}\le  4nR^2\sqrt{m} +  B\norm{\vZ(0)}_F .\]
%\end{thm}
%
%\begin{comment}
%\begin{thm}
%Suppose $\{\vx_i\}_{i=1}^n$ are sampled indepedently from distribution $\cD$, the support of  which is a subset of $\{\vx\in\reals^d \mid \norm{\vx}_2\le 1\}$, and that $\{\vw_r(0)\}_{r=1}^m$,$\{a_r\}_{r=1}^m$ are independently sampled from $N(\bm{0}, I_d)$, $\textrm{Unif}[-1,1]$ respectively. Then w.h.p. (over both of the randomness of $X$,$W(0)$,$\va$), the empirical rademacher complexity has the following upper bound 
% \[\rc{\cF^{\vW(0),\va}_{R,B}; \vX}\le  4nR^2\sqrt{m} +  \norm{\vZ(0)}_F .\]
%\end{thm}
%\end{comment}

\begin{proof}[Proof of Lemma~\ref{lem:rad_dist_func_class}]
	We need to upper bound
	\begin{align*}
		\calR_S\left(\cF^{\mat{W}(0),\va}_{R,B} \right) 
		&= \frac{1}{n} \Erc{\sup_{f \in \cF^{\mat{W}(0),\va}_{R,B}}\sum_{i=1}^{n} \eps_i f(\vect{x}_i)} \\
		&= \frac{1}{n} \Erc{ \sup_{\mat W: \norm{\mat W - \mat W(0)}_{2, \infty} \le R \atop \norm{\mat W - \mat W(0)}_F \le B }\sum_{i=1}^{n} \eps_i f_{\mat W, \vect a}(\vect{x}_i)} \\
		&= \frac{1}{n} \Erc{ \sup_{\mat W: \norm{\mat W - \mat W(0)}_{2, \infty} \le R \atop \norm{\mat W - \mat W(0)}_F \le B }\sum_{i=1}^{n} \eps_i \sum_{r=1}^m \frac{1}{\sqrt m} a_r \sigma(\vect w^\top \vect x_i)} ,
	\end{align*}
where $\norm{\mat W - \mat W(0)}_{2, \infty} = \max\limits_{r\in[m]} \norm{\vw_r-\vw_r(0)}_2$.

Similar to the proof of Lemma~\ref{lem:H-and-Z-perturbation}, we define events:
\[A_{r, i} = \left\{ \abs{\vw_{r}(0)^\top \vx_i} \le R \right\}, \quad i\in[n], r\in[m].\]
Since we only look at $\mat W$ such that $\norm{\vw_r-\vw_r(0)}_2\le R$ for all $ r\in[m]$,
if $\indict\{A_{r, i}\}=0$ we must have $\indict\{ \vw^\top \vx_i \} = \indict\{\vw_r(0)^\top \vx_i\ge  0\} = \indict_{r, i}(0)$. Thus we have
\[ \bone{\neg A_{r, i}} \relu{\vw_r^\top \vx_i} %=  \bone{\neg A_{r, i}} \vw_r^\top \vx_i \bone{\vw_r(0)^\top \vx_i\ge  0}  
=  \bone{\neg A_{r, i}} \indict_{r, i}(0)\vw_r^\top \vx_i. \]
Then we have
\begin{align*}
%\begin{aligned}
&\sum_{i=1}^n \eps_i \sum_{r=1}^m a_r \relu{\vw_r^\top \vx_i} -\sum_{i=1}^n \eps_i \sum_{r=1}^m a_r \indict_{r,i}(0) \vw_r^\top \vx_i   \\
=\,& \sum_{r=1}^m \sum_{i=1}^n \left( \bone{A_{r, i}} + \bone{\neg A_{r, i}}\right) \eps_i a_r \left( \relu{\vw_r^\top \vx_i} - \indict_{r,i}(0) \vw_r^\top \vx_i   \right)\\
=\,& \sum_{r=1}^m \sum_{i=1}^n  \bone{A_{r, i}}  \eps_i a_r \left( \relu{\vw_r^\top \vx_i} - \indict_{r,i}(0) \vw_r^\top \vx_i   \right)\\
=\,& \sum_{r=1}^m \sum_{i=1}^n  \bone{A_{r, i}}  \eps_i a_r \left( \relu{\vw_r^\top \vx_i} - \indict_{r,i}(0) \vw_r(0)^\top \vx_i    - \indict_{r,i}(0) (\vw_r-\vw_r(0))^\top \vx_i  \right)\\
=\,& \sum_{r=1}^m \sum_{i=1}^n  \bone{A_{r, i}}  \eps_i a_r \left( \relu{\vw_r^\top \vx_i} - \relu{\vw_r(0)^\top \vx_i}   - \indict_{r,i}(0) (\vw_r-\vw_r(0))^\top \vx_i   \right)\\
\le\, & \sum_{r=1}^m \sum_{i=1}^n  \bone{A_{r, i}} \cdot 2R  .
%\end{aligned}
\end{align*}

Thus we can bound the Rademacher complexity as:
\begin{align*}
%\begin{split}
\calR_S\left( \cF^{\vW(0),\va}_{R,B} \right)
=\,& \frac1n \Erc{ \sup_{\tiny \substack{\mat W: \norm{\vW-\vW(0)}_{2,\infty}\le R \\ \norm{\vW-\vW(0)}_F\le B }}  \sum_{i=1}^n\eps_i \sum_{r=1}^m\frac{ a_r}{\sqrt{m}} \relu{\vw_r^\top \vx} }\\
\le\, & \frac1n  \Erc{ \sup_{\tiny \substack{\mat W: \norm{\vW-\vW(0)}_{2,\infty}\le R \\ \norm{\vW-\vW(0)}_F\le B }}  \sum_{i=1}^n\eps_i  \sum_{r=1}^m\frac{ a_r}{\sqrt{m}} \indict_{r,i}(0) \vw_r^\top \vx_i   } + \frac{2R}{n\sqrt{m}}\sum_{r=1}^m \sum_{i=1}^n  \bone{A_{r, i}}  \\
\le\, & \frac1n  \Erc{ \sup_{\mat W: \norm{\vW-\vW(0)}_F\le B}  \sum_{i=1}^n\eps_i   \sum_{r=1}^m\frac{ a_r}{\sqrt{m}} \indict_{r,i}(0) \vw_r^\top \vx_i  } + \frac{2R}{n\sqrt{m}}\sum_{r=1}^m \sum_{i=1}^n  \bone{A_{r, i}}  \\
=\, & \frac1n  \Erc{ \sup_{\mat W: \norm{\vW-\vW(0)}_F\le B}  \vectorize{\mat W}^\top \mat Z(0) {\bm{\eps}} }+ \frac{2R}{n\sqrt{m}}\sum_{r=1}^m \sum_{i=1}^n  \bone{A_{r, i}}  \\
=\, & \frac1n  \Erc{ \sup_{\mat W: \norm{\vW-\vW(0)}_F\le B}  \vectorize{\mat W - \mat W(0)}^\top \mat Z(0) {\bm{\eps}} } + \frac{2R}{n\sqrt{m}}\sum_{r=1}^m \sum_{i=1}^n  \bone{A_{r, i}}  \\
\le\,& \frac1n \Erc{B\cdot \norm{\mat Z(0) \bm{\eps}}_2} + \frac{2R}{n\sqrt{m}}\sum_{r=1}^m \sum_{i=1}^n  \bone{A_{r, i}}  \\
\le\,& \frac Bn \sqrt{\Erc{ \norm{\mat Z(0) \bm{\eps}}_2^2}} + \frac{2R}{n\sqrt{m}}\sum_{r=1}^m \sum_{i=1}^n  \bone{A_{r, i}}  \\
=\, & \frac Bn \norm{\vZ(0)}_F + \frac{2R}{n\sqrt{m}} \sum_{r=1}^m \sum_{i=1}^n  \bone{A_{r, i}} .
%\end{split}
\end{align*}

Next we bound $\norm{\mat Z(0)}_F$ and $\sum_{r=1}^m \sum_{i=1}^n  \bone{A_{r, i}}$.

For $\norm{\mat Z(0)}_F$, notice that
\begin{align*}
\norm{\mat Z(0)}_F^2 = \frac1m \sum_{r=1}^m \left( \sum_{i=1}^n \indict_{r,i}(0) \right).
\end{align*}
Since all $m$ neurons are independent at initialization and $\E\left[ \sum_{i=1}^n \indict_{r,i}(0) \right] = n/2$, by Hoeffding's inequality, with probability at least $1-\delta/2$ we have
 \begin{align*}
 \norm{\mat Z(0)}_F^2 \le n\left( \frac 12 + \sqrt{\frac{\log \frac2\delta}{2m}} \right).
 \end{align*}

Similarly, for $\sum_{r=1}^m \sum_{i=1}^n  \bone{A_{r, i}}$, from \eqref{eqn:A_{ri}-bound} we know $\E\left[ \sum_{i=1}^n  \bone{A_{r, i}} \right] \le \frac{\sqrt 2 n R}{\sqrt\pi \kappa}$. Then by Hoeffding's inequality, with probability at least $1-\delta/2$ we have
\begin{align*}
\sum_{r=1}^m \sum_{i=1}^n  \bone{A_{r, i}} \le mn \left( \frac{\sqrt 2  R}{\sqrt\pi \kappa} + \sqrt{\frac{\log \frac2\delta}{2m}} \right).
\end{align*}

Therefore, with probability at least $1-\delta$, the Rademacher complexity is bounded as:
\begin{align*}
\calR_S\left( \cF^{\vW(0),\va}_{R,B} \right)
&\le \frac Bn \left( \sqrt{\frac n2} + \sqrt{n\sqrt{\frac{\log \frac2\delta}{2m}}} \right) + \frac{2R}{n\sqrt{m}} mn \left( \frac{\sqrt 2  R}{\sqrt\pi \kappa} + \sqrt{\frac{\log \frac2\delta}{2m}} \right) \\
&= \frac{B}{\sqrt{2n}} \left(1+\left( \frac{2\log \frac2\delta}{m} \right)^{1/4} \right)  + \frac{2\sqrt 2 R^2 \sqrt m}{\sqrt\pi \kappa} + R \sqrt{2\log \frac2\delta},
\end{align*}
completing the proof of Lemma~\ref{lem:rad_dist_func_class}.
(Note that the high probability events used in the proof do not depend on the value of $B$, so the above bound holds simultaneously for every $B$.)
\end{proof}

\subsection{Proof of Theorem~\ref{thm:main_generalization}} \label{app:proof-thm:main_generalization}

\begin{proof}[Proof of Theorem~\ref{thm:main_generalization}]
	First of all, since the distribution $\cD$ is $(\lambda_0, \delta/3, n)$-non-degenerate, with probability at least $1-\delta/3$ we have $\lambda_{\min}(\mat H^\infty) \ge \lambda_0$. The rest of the proof is conditioned on this happening.
	
	 Next, from Theorem~\ref{thm:ssdu-converge}, Lemma~\ref{lem:distance_bounds} and Lemma~\ref{lem:rad_dist_func_class}, we know that for any sample $S$, with probability at least $1-\delta/3$ over the random initialization, the followings hold simultaneously:
	 \begin{enumerate}[(i)]
	 	\item Optimization succeeds (Theorem~\ref{thm:ssdu-converge}):
	 	\begin{equation*} 
	 	\Phi(\mat W(k)) \le \left( 1-\frac{\eta\lambda_0}{2} \right)^k \cdot O\left( \frac{n}{\delta} \right) \le \frac12. %\frac{n\epsilon^2}{8}.
	 	\end{equation*}
	 	This implies an upper bound on the training error $L_S(f_{\mat W(k), \vect a}) = \frac1n \sum_{i=1}^n \ell(f_{\mat W(k), \vect a}(\vect x_i), y_i) = \frac1n \sum_{i=1}^n \ell(u_i(k), y_i)$:
	 	\begin{align*}
	 	L_S(f_{\mat W(k), \vect a}) &= \frac1n \sum_{i=1}^n \left[ \ell(u_i(k), y_i) - \ell(y_i, y_i)\right] \\
	 	&\le \frac1n \sum_{i=1}^n \abs{u_i(k) - y_i} \\
	 	&\le \frac{1}{\sqrt n} \norm{\vect u(k) - \vect y}_2 \\
	 	&= \sqrt{\frac{2\Phi(\mat W(k))}{n}} \\
	 	&\le \frac{1}{\sqrt n}. %\frac{\epsilon}{2}.
	 	\end{align*}
	 	
	 	\item $\norm{\vect w_r(k) - \vect w_r(0)}_2 \le R$ $(\forall r\in[m])$ and
	 	%$\norm{\mat{W}(k) - \mat{W}(0)}_{2,\infty} \le \frac{\poly(n,1/\lambda_0)}{\sqrt{m}}$,
	 	 $\norm{\mat{W}(k)-\mat{W}(0)}_{F} \le B $,
	 	%The learned function $f_{\mat W(k), \vect a}$ belongs to the function class $\cF_{R, B}^{\mat W(0), \vect a}$, 
	 	where $R = O\left( \frac{ n }{\sqrt m \lambda_0\sqrt{\delta}} \right)$   and $B = \sqrt{\vect{y}^\top \left(\mat{H}^{\infty}\right)^{-1}\vect{y}} + O\left( \frac{  n \kappa}{\lambda_0 \delta} \right) + \frac{\poly\left(n, \lambda_0^{-1}, \delta^{-1} \right)}{m^{1/4} \kappa^{1/2}}$.
	 	Note that $B\le O\left( \sqrt{\frac{n}{\lambda_0}} \right)$.
	 	
	 	\item 
	 	Let $B_i =  i$ ($i=1, 2, \ldots$).
	 	Simultaneously for all $i$, the function class $\cF_{R, B_i}^{\mat W(0), \vect a}$ has Rademacher complexity bounded as
	 	\begin{align*}
	 	\calR_S\left( \cF^{\mat W(0),\va}_{R,B_i} \right)
	 	\le 	\frac{B_i}{\sqrt{2n}} \left(1+\left( \frac{2\log \frac{10}{\delta}}{m} \right)^{1/4} \right) + \frac{2 R^2 \sqrt m}{ \kappa} + R \sqrt{2\log \frac{10}{\delta}} .
	 	\end{align*}
	 \end{enumerate}

 Let $i^*$ be the smallest integer such that $B\le B_{i^*}$. Then we have $i^* \le O\left( \sqrt{\frac{n}{\lambda_0}} \right)$ and $B_{i^*} \le B + 1$.
 From above we know $f_{\mat W(k), \vect a}\in \cF_{R, B_{i^*}}^{\mat W(0), \vect a}$, and
 \begin{align*}
 %\begin{aligned}
 &\calR_S\left( \cF^{\mat W(0),\va}_{R,B_{i^*}} \right) \\
 \le \,&	\frac{B+1}{\sqrt{2n}} \left(1+\left( \frac{2\log \frac{10}{\delta}}{m} \right)^{1/4} \right) + \frac{2 R^2 \sqrt m}{ \kappa} + R \sqrt{2\log \frac{10}{\delta}} \\
 = \,& \frac{\sqrt{\vect{y}^\top \left(\mat{H}^{\infty}\right)^{-1}\vect{y}} }{\sqrt{2n}} \left(1+\left( \frac{2\log \frac{10}{\delta}}{m} \right)^{1/4} \right) + \frac{1}{\sqrt{n}} + O\left( \frac{ \sqrt n \kappa}{\lambda_0 \delta} \right) + \frac{\poly\left(n, \lambda_0^{-1}, \delta^{-1} \right)}{m^{1/4} \kappa^{1/2}} + \frac{2 R^2 \sqrt m}{ \kappa} + R \sqrt{2\log \frac{10}{\delta}} \\
 \le \,& \sqrt{\frac{\vect{y}^\top \left(\mat{H}^{\infty}\right)^{-1}\vect{y}}{2n}}  +  \sqrt{\frac{\sqrt n \lambda_0^{-1}\sqrt n}{2n}} \left( \frac{2\log \frac{10}{\delta}}{m} \right)^{1/4} + \frac{1}{\sqrt{n}}  + O\left( \frac{ \sqrt n \kappa}{\lambda_0 \delta} \right) + \frac{\poly\left(n, \lambda_0^{-1}, \delta^{-1} \right)}{m^{1/4} \kappa^{1/2}} \\
 = \,& \sqrt{\frac{\vect{y}^\top \left(\mat{H}^{\infty}\right)^{-1}\vect{y}}{2n}}  + \frac{1}{\sqrt{n}}  + O\left( \frac{ \sqrt n \kappa}{\lambda_0 \delta} \right) + \frac{\poly\left(n, \lambda_0^{-1}, \delta^{-1} \right)}{m^{1/4} \kappa^{1/2}} \\
 \le \,& \sqrt{\frac{\vect{y}^\top \left(\mat{H}^{\infty}\right)^{-1}\vect{y}}{2n}}   + \frac{2}{\sqrt{n}} .
 %\end{aligned}
 \end{align*}

 Next, from the theory of Rademacher complexity (Theorem~\ref{thm:rad_generalization}) and a union bound over a finite set of different $i$'s, for any random initialization $\left(\mat W(0), \mat a \right)$, with probability at least $1-\delta/3$ over the sample $S$, we have
 \begin{equation*}
 \sup_{f \in \cF_{R, B_i}^{\mat W(0), \vect a}} \left\{ L_{\calD}(f)-L_{S}(f) \right\} \le 2 \calR_S\left(\cF_{R, B_i}^{\mat W(0), \vect a}\right) + O\left( \sqrt{\frac{\log\frac{n}{\lambda_0\delta}}{n}} \right) , \qquad \forall i\in \left\{1, 2, \ldots, O\left( \sqrt{\frac{n}{\lambda_0}} \right) \right\}.
 \end{equation*}
 
 Finally, taking a union bound, we know that with probability at least $1-\frac23\delta$ over the sample $S$ and the random initialization $(\mat W(0), \vect a)$, the followings are all satisfied (for some $i^*$):
 \begin{equation*}
 \begin{aligned}
 L_S(f_{\mat W(k), \vect a}) &\le \frac{1}{\sqrt n},\\
 f_{\mat W(k), \vect a} &\in \cF_{R, B_{i^*}}^{\mat W(0), \vect a} ,\\
 \calR_S\left( \cF^{\mat W(0),\va}_{R,B_{i^*}} \right) &\le \sqrt{\frac{\vect{y}^\top \left(\mat{H}^{\infty}\right)^{-1}\vect{y}}{2n}}  + \frac{2}{\sqrt{n}}  , \\
  \sup_{f \in \cF_{R, B_{i^*}}^{\mat W(0), \vect a}} \left\{ L_{\calD}(f)-L_{S}(f) \right\} &\le 2 \calR_S\left(\cF_{R, B_{i^*}}^{\mat W(0), \vect a}\right) + O\left( \sqrt{\frac{\log\frac{n}{\lambda_0\delta}}{n}} \right)  .
 \end{aligned}
 \end{equation*}
 These together can imply:
 \begin{align*}
 L_\cD(f_{\mat W(k), \vect a}) &\le \frac{1}{\sqrt n} + 2 \calR_S\left(\cF_{R, B_{i^*}}^{\mat W(0), \vect a}\right) + O\left( \sqrt{\frac{\log\frac{n}{\lambda_0\delta}}{n}} \right) \\
 &\le \frac{1}{\sqrt n} + 2 \left(\sqrt{\frac{\vect{y}^\top \left(\mat{H}^{\infty}\right)^{-1}\vect{y}}{2n}}  + \frac{2}{\sqrt{n}}    \right) + O\left( \sqrt{\frac{\log\frac{n}{\lambda_0\delta}}{n}} \right)  \\
  &= \sqrt{\frac{2\vect{y}^\top \left(\mat{H}^{\infty}\right)^{-1}\vect{y}}{n}}  + O\left( \sqrt{\frac{\log\frac{n}{\lambda_0\delta}}{n}} \right) .
 \end{align*}
 This completes the proof.
\end{proof}

\subsection{Proof of Corollary~\ref{cor:binary-classification-generalization}} \label{app:proof-cor:binary-classification-generalization}

\begin{proof}[Proof of Corollary~\ref{cor:binary-classification-generalization}]
	We apply Theorem~\ref{thm:main_generalization} to the ramp loss
	\begin{align*}
		 \ell^{\mathrm{ramp}}(u, y) =
		 \begin{cases}
		 1, & uy\le0,\\
		 1-uy, & 0<uy<1,\\
		 0, & uy\ge1.
		 \end{cases} \qquad (u\in\R, y\in\{\pm1\})
	\end{align*}
	Note that it is $1$-Lipschitz in $u$ for $y\in\{\pm1\}$ and satisfies $\ell^{\mathrm{ramp}}(y, y) = 0$ for $y\in\{\pm1\}$. It is also an upper bound on the 0-1 loss:
	\[
	\ell^{\mathrm{ramp}}(u, y) \ge \ell^{01}(u, y) = \indict\left\{ uy\le0 \right\}.
	\]
	
	Therefore we have with probability at least $1-\delta$:
	\begin{align*}
	L_\calD^{01}(f_{\mat W(k), \vect a}) 
	&= \E_{(\vect x, y)\sim \cD} \left[ \ell^{01}(f_{\mat W(k), \vect a}(\vect x), y)  \right] \\
	&\le \E_{(\vect x, y)\sim \cD} \left[ \ell^{\mathrm{ramp}}(f_{\mat W(k), \vect a}(\vect x), y)  \right] \\
	&\le \sqrt{\frac{2 \vect y^\top (\mat H^\infty)^{-1}\vect y }{n}} + O\left( \sqrt{\frac{\log\frac{n}{\lambda_0\delta}}{n}} \right)  .
	\qedhere
	\end{align*}
\end{proof}

\section{Proofs for Section~\ref{sec:improper}}
\label{app:proof_improper}

We prove a lemma before proving Theorem~\ref{thm:improper_learning_monomial}.

\begin{lem}\label{lem:inv_comparison}
For any two symmetric matrices $\vect{A},\vect{B}\in \reals^{n\times n}$ such that $\vect{B}\succeq \vect{A}\succeq \mat0$, we have $\vect{A}^\dag\succeq \vect{P}_{\mat A} \vect{B}^\dag \vect{P}_{\mat A}$, where $\vect{P}_\vect{A} = \vect{A}^{1/2}\vect{A}^\dag \vect{A}^{1/2}$ is the projection matrix for the subspace spanned by $\vect{A}$.\footnote{$\vect{A}^\dag$ is the Moore-Penrose pseudo-inverse of $\vect{A}$.}

Note that when $\vect{A}$ and $\vect{B}$ are both invertible,  this result reads $\vect{A}^{-1}\succeq \vect{B}^{-1}$.
\end{lem}

\begin{proof}
W.L.O.G. we can assume $\vect{B}$ is invertible, which means $\vect{B}^{-1}=\vect{B}^{\dag}$. Additionally, we can assume $\vect{A}$ is diagonal and $\vect{A}= \begin{pmatrix}\mat \Lambda &\mat0 \\\mat0& \mat0\end{pmatrix} $. Thus $\vect{P}_
\vect{A} =  \begin{pmatrix} \mat I & \mat 0 \\\mat0& \mat0\end{pmatrix}$. We define $\vect{P}_\vect{A}^\perp = \mat I-\vect{P}_\vect{A} =\begin{pmatrix} \mat0 & \mat0 \\ \mat0& \mat I\end{pmatrix}$. 

Now we will show that all the solutions to the equation $\det(\vect{P}_\vect{A}\vect{B}^{-1}\vect{P}_\vect{A}-\lambda (\vect{A}^\dag + \vect{P}_\vect{A}^\perp))=0$ are between $0$ and $1$.
If this is shown, we must have $\vect{P}_\vect{A}\vect{B}^{-1}\vect{P}_\vect{A} \preceq \vect{A}^\dag + \vect{P}_\vect{A}^\perp$, which would imply $\vect{P}_\vect{A}\vect{B}^{-1}\vect{P}_\vect{A} \preceq \vect{A}^\dag$. 

We have
\begin{equation}
\begin{split}
&\det(\vect{P}_\vect{A}\vect{B}^{-1}\vect{P}_\vect{A}-\lambda (\vect{A}^\dag + \vect{P}_\vect{A}^\perp)) \\
=& \det(\vect{B}^{-1}\vect{P}_\vect{A}-\lambda (\vect{A}^\dag + \vect{P}_\vect{A}^\perp)) \\
=& \det(\vect{B}^{-1})\det(\vect{P}_\vect{A}-\lambda \vect{B}(\vect{A}^\dag + \vect{P}_\vect{A}^\perp))\\
= &\det(\vect{B}^{-1})\det((\vect{A}-\lambda \vect{B})(\vect{A}^\dag + \vect{P}_\vect{A}^\perp))\\ 
= &\det(\vect{B}^{-1})\det(\vect{A}^\dag + \vect{P}_\vect{A}^\perp)\det(\vect{A}-\lambda \vect{B}).
\end{split}
\end{equation}
Note $\det(\vect{B}^{-1})>0$ and $\det(\vect{A}^\dag + \vect{P}_\vect{A}^\perp)>0$.
Thus all the solutions to  $\det(\vect{P}_\vect{A}\vect{B}^{-1}\vect{P}_\vect{A}-\lambda (\vect{A}^\dag + \vect{P}_\vect{A}^\perp))=0$ are exactly all the solutions to $\det(\vect{A}-\lambda \vect{B})=0$. Since $\vect{B}\succeq \vect{A}\succeq \mat0$ and $\mat B \succ \mat0$, we have $\det(\vect{A}-\lambda \vect{B})\neq 0$ when $\lambda<0$ or $\lambda >1$.
\end{proof}

\begin{proof}[Proof of Theorem~\ref{thm:improper_learning_monomial}]
	
	For vectors $\vect{a} = (a_1, \ldots, a_{n_1})^\top \in \reals^{n_1}, \vect{b} = (b_1, \ldots, b_{n_2}) \in\reals^{n_2}$, the \emph{tensor product} of $\vect{a}$ and $\vect{b}$ is defined as $\vect{a}\otimes\vect{b}\in\reals^{n_1 n_2}$, where 
	$[\vect{a}\otimes\vect{b}]_{(i-1)n_2+j} = a_ib_j $.
	For matrices $\vect{A} = (\vect a_1, \ldots, \vect a_{n_3}) \in\reals^{n_1\times n_3}, \vect{B} = (\vect b_1, \ldots, \vect b_{n_3}) \in\reals^{n_2\times n_3}$,  the \emph{Khatri-Rao product} of $\vect{A}$ and $\vect{B}$ is defined as  $\vect{A}\odot \vect{B}\in\reals^{n_1n_2\times n_3}$, where $\vect{A}\odot \vect{B} = [\vect{a}_1\otimes \vect{b}_1, \vect{a}_2\otimes \vect{b}_2, \ldots, \vect{a}_{n_3}\otimes \vect{b}_{n_3}]$. We use $\vect{A}\circ \vect{B}$ to denote the \emph{Hadamard product} (entry-wise product) of matrices $\vect{A}$ and $\vect{B}$ of the same size, i.e., $[\mat A \circ \mat B]_{ij} = \mat A_{ij} \mat B_{ij}$.
	We also denote their corresponding powers by $\vect a^{\otimes l}$, $\mat A^{\odot l}$ and $\mat A^{\circ l}$.
	
	Recall $\mat X = (\vect x_1, \ldots, \vect x_n) \in \R^{d\times n}$.
	Let  $\vect{K} = \vect{X}^\top \vect{X} \in \reals^{n \times n}$.
	Then it is easy to see that $[\vect{K}^{\circ l}]_{ij} = \vect{K}^l_{ij} = \inp{\vx_i}{\vx_j}^l$ and $\vect{K}^{\circ l} = (\vX^{\odot l})^\top \vX^{\odot l}\succeq \bm{0}$ for all $l\in\mathbb N$.
	Recall from~\eqref{eqn:H_infy_defn} that $\vect{H}^\infty_{ij} = \frac{\mat K_{ij}}{4}+ \frac{\vect{K}_{ij} \arcsin(\vect{K}_{ij})}{2\pi}$. 
	Since  $\arcsin(x) = \sum_{l=0}^\infty \frac{(2l-1)!!}{(2l)!!}\cdot \frac{x^{2l+1}}{2l+1}\, (|x|\le1)$\footnote{$p!! = p(p-2)(p-4)\cdots$ and $0!!=(-1)!!=1$.},  we have 
	\[ \vect{H}^\infty_{ij} = \frac{\vect{K}_{ij}}{4}+ \frac{1}{2\pi} \sum_{l=1}^\infty \frac{(2l-3)!!}{(2l-2)!!}\cdot \frac{\vect{K}_{ij}^{2l}}{2l-1},\]
	which means
	\[ \vect{H}^\infty = \frac{\vect{K}}{4}+ \frac{1}{2\pi} \sum_{l=1}^\infty \frac{(2l-3)!!}{(2l-2)!!}\cdot \frac{\vect{K}^{\circ 2l}}{2l-1}.\]
	
	Since $\vect{K}^{\circ l}\succeq \bm{0}$ ($\forall l\in \mathbb{N}$), we have 
	$ \vect{H}^\infty \succeq \frac{\vect{K}}{4}$ and 
	\[  \vect{H}^\infty \succeq \frac{1}{2\pi} \frac{(2l-3)!!}{(2l-2)!!}\cdot \frac{\vect{K}^{\circ 2l}}{2l-1} \succeq \frac{\vect{K}^{\circ 2l}}{2\pi(2l-1)^2}, \quad \forall l\in\mathbb N_+.\]

	Now we proceed to prove the theorem.
	
	%We decompose the label vector into $\vect y = \vect y^{(1)} + \sum_{l=1}^\infty \vect y^{(2l)}$, where $\left[ \vect y^{(p)} \right]_i = \alpha_p (\vbeta_p^\top \vect x_i)^p$ for $p\in\{1, 2, 4, 6, \ldots\}$ and $i\in[n]$.
	
	First we consider the case $p=1$. In this case we have $\vy =\alpha \vX^\top \vbeta_1$. Since $\mat H^\infty \succeq \frac{\mat K}{4}$,
	 from Lemma~\ref{lem:inv_comparison} we have 
	\[ \vect{P}_{\vect{K}} (\vect{H}^{\infty})^{-1} \vect{P}_{\vect{K}} \preceq 4\vect{K}^\dag,\]
	where $\vect{P}_\vect{K} = \vect{K}^{1/2}\vect{K}^\dag \vect{K}^{1/2}$ is the projection matrix for the subspace spanned by $\vect{K}$. Since $\vect{K} = \vect{X}^\top \vect{X}$, we have $\vect{P}_{\vect{K}}\vX^\top = \vX^\top$. 
	%$\vect{K}$ has the same column space as $\mat X$, which implies $\vect{P}_{\vect{K}}\vX^\top = \vX^\top$. 
	Therefore, we have
	\begin{align*}
	&\vy^\top (\vect{H}^{\infty})^{-1}  \vy\\
	 =\, &\alpha^2 \vbeta^\top \vX (\vect{H}^{\infty})^{-1}  \vX^\top \vbeta \\
	=\,&\alpha^2  \vbeta^\top \vX \vect{P}_{\vect{K}}(\vect{H}^{\infty})^{-1} \vect{P}_\vect{K} \vX^\top \vbeta\\
	 \le\, & 4\alpha^2 \vbeta^\top \vX \vect{K}^\dag \vX^\top \vbeta  \\
	 =\, &4\alpha^2 \vbeta^\top \vect{P}_{\mat X \mat X^\top}\vbeta\\
	 \le\,& 4\alpha^2 \norm{\vbeta}_2^2.
	\end{align*}
	This finishes the proof for $p=1$.

	 Similarly, for $p=2l$ $(l\in\mathbb N_+)$, we have $\vy  = \alpha \left( \vX^{\odot 2l} \right)^\top \vbeta^{\otimes 2l}$. From
	 $ \vect{H}^{\infty} \succeq \frac{\vect{K}^{\circ 2l}}{2\pi(2l-1)^2} = \frac{(\vX^{\odot 2l})^\top \vX^{\odot 2l}}{2\pi (2l-1)^2}$ and Lemma~\ref{lem:inv_comparison} we have
	 \begin{align*}
	&\vy^\top (\vect{H}^{\infty})^{-1}  \vy\\
	 =\, & \alpha^2 (\vbeta^{\otimes 2l})^\top\vX^{\odot 2l} (\vect{H}^{\infty})^{-1}    (\vX^{\odot 2l})^\top \vbeta^{\otimes 2l}\\
	 \le\, & 2\pi(2l-1)^2  \alpha^2  (\vbeta^{\otimes 2l})^\top\vX^{\odot 2l} (\vect{K}^{\circ 2l})^{\dag}    (\vX^{\odot 2l})^\top \vbeta^{\otimes 2l}\\
	 =\, &2\pi(2l-1)^2 \alpha^2    (\vbeta^{\otimes 2l})^\top \vect{P}_{\vX^{\odot 2l} (\vX^{\odot 2l})^\top} \vbeta^{\otimes 2l}\\
	 \le\, & 2\pi (2l-1)^2  \alpha^2   \norm{\vbeta^{\otimes 2l}}_2^2\\
	 =\, &2\pi (2l-1)^2 \alpha^2  \norm{\vbeta}_2^{4l}\\
	 \le\,& 2\pi p^2 \alpha^2  \norm{\vbeta}_2^{2p}.
	 \end{align*}
	 This finishes the proof for $p=2l$.
\end{proof}

\end{document}